%% file: main.tex
\PassOptionsToPackage{unicode}{hyperref}
\PassOptionsToPackage{hyphens}{url}
\PassOptionsToPackage{dvipsnames,svgnames,x11names}{xcolor}

\documentclass[12pt]{article}

\usepackage[mathlines]{lineno}

\usepackage{setspace}
\usepackage[shortlabels]{enumitem}

\usepackage[utf8]{inputenc}
\usepackage[T1]{fontenc}
\usepackage[english]{babel}

\usepackage{amsmath,amssymb,amsfonts}
\usepackage{amsthm}
\usepackage{mathtools}
\usepackage{bm}

\usepackage{algorithm}
\usepackage{algpseudocode}

\usepackage{graphicx}
\usepackage{booktabs}
\usepackage{multirow}
\usepackage{caption}
\usepackage{subcaption}
\usepackage{makecell}

\usepackage{authblk}

\usepackage[round,authoryear]{natbib}
\usepackage{xcolor}

\usepackage{geometry}
\usepackage[colorlinks=true,allcolors=blue]{hyperref}

\newtheorem{theorem}{Theorem}[section]
\newtheorem{lemma}[theorem]{Lemma}
\newtheorem{proposition}[theorem]{Proposition}
\newtheorem{corollary}[theorem]{Corollary}
\theoremstyle{definition}

\newtheorem{example}[theorem]{Example}
\theoremstyle{remark}
\newtheorem{remark}[theorem]{Remark}

\DeclareMathOperator{\Var}{Var}

\DeclareMathOperator*{\argmin}{argmin}
\newcommand{\E}{\mathbb{E}}
\DeclareMathOperator{\supp}{supp}
\DeclareMathOperator{\sgn}{sgn}

\hypersetup{
colorlinks=true,
linkcolor=black, 
urlcolor=black,   
citecolor=black                            
}

\begin{document}

\title{ Supervised Learning with In-Context Prediction}
\author[1]{Erin Craig\thanks{Corresponding author: \href{mailto:ercr@umich.edu}{ercr@umich.edu}}}
\author[2]{Sourav Chatterjee}
\author[3,2]{Rob Tibshirani}
\affil[1]{Department of Biostatistics, University of Michigan, Ann Arbor, MI, USA}
\affil[2]{Department of Statistics, Stanford University, Stanford, CA, USA}
\affil[3]{Department of Biomedical Data Science, Stanford University, Stanford, CA, USA}

\date{}
\maketitle

\begin{abstract}%
In‑context learning with attention allows transformers to make personalized predictions by attending to relevant training examples, flexibly adapting to heterogeneous data. This inspired us to revisit local regression. 
When predicting for a test observation, in-context learning upweights similar training examples, where ``similarity'' can be defined in a supervised manner; local regression weights typically rely only on covariates and may be suboptimal.

This paper introduces a supervised learning framework inspired by in-context learning with simplified attention. As a leading example, we propose {\em attention lasso}, which fits a personalized lasso for each prediction point by weighting training examples using a supervised similarity metric. To control complexity, we use a shared lasso penalty and blend local models with a global model via a cross‑validated parameter. This framework works with other learners, like gradient boosting, and extends to time series and spatial data. We give theory for a hard-attention variant that fits the lasso in the nodes of a single CART tree: in a sparse, continuous, piecewise affine model, it recovers the support and coefficient signs of each query's local linear model, and when the regression function is not affine it has lower prediction risk than every global affine fit, including the lasso.
\end{abstract}

\noindent%

\section{Introduction}

Consider a dataset consisting of features $\bm{x}_1$ and $\bm{x}_2$ and response $\bm{y}$. Our goal is to predict $\bm{y}$ given $\bm{x}_1$ and $\bm{x}_2$. Further suppose that, unknown to us, the data are heterogeneous: observations fall into two subgroups with different covariate-response relationships, as illustrated in Figure~\ref{fig:motivation}. 

\begin{figure}[H]
    \centering
    \includegraphics[width=.7\linewidth]{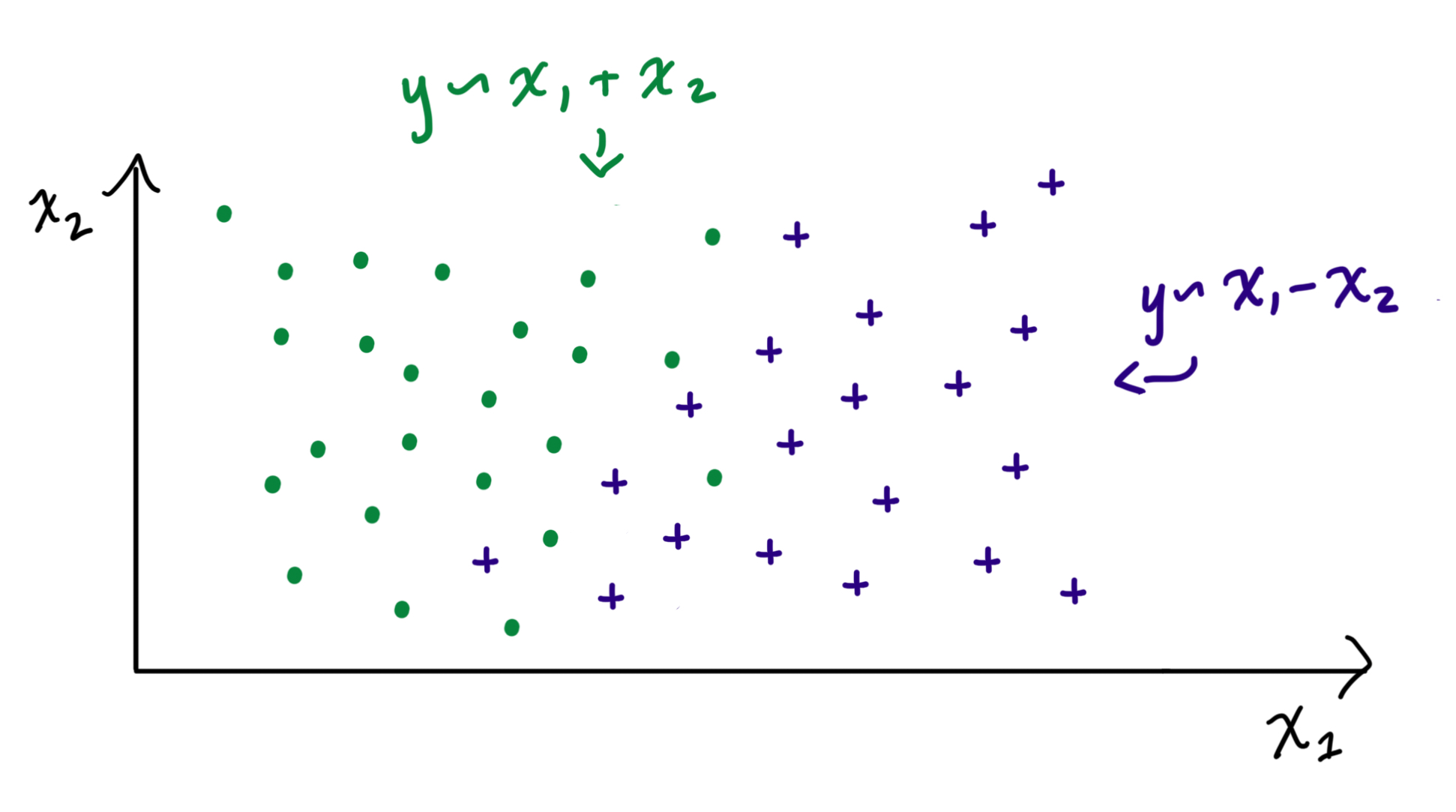}
    \caption{\em Toy example of a dataset with two features, $x_1$ and $x_2$, where the true model depends on the values of the features.}
    \label{fig:motivation}
\end{figure}


Our setup here is a common statistical framework: our data is composed of a feature matrix $\bm{X}$ and a vector response $\bm{y}$. This has become known as \emph{tabular data} in the machine learning community, in contrast to sequence data (e.g. text) that is the focus of large language models and has led to the development of attention and in-context learning. 

If we use a standard linear model (fit by say least squares or lasso), we will fit a single set of coefficients to all observations, thereby averaging effects across the two subgroups. If we suspected heterogeneity in our data, we might cluster the data and then fit separate models to each resulting group. This could be successful if the clusters are well separated and we choose the correct number of clusters. Another option is to fit a more complicated model like gradient boosting or a neural network, which adapt to data heterogeneity but sacrifice interpretability.

In this paper we propose a method to fit personalized, locally weighted models for each test point, anchored by a global model fit to the full training set.  At a high level, for a test observation $\bm{x}^*$ and training data $(\bm{X}, \bm{y})$, our method (1) derives a set of observation weights $\bm{w}^*$ (\textit{attention weights}) that describe similarity between $\bm{x}^*$ and the rows of $\bm{X}$ as they relate to the response $\bm{y}$; and (2) uses them to fit a ``bespoke'' weighted model to make a prediction at $\bm{x}^*$. 

Importantly, we define similarity in a \emph{supervised} way: two observations are considered similar only if they share features and interactions useful for predicting $\bm{y}$, which underlies the improvement in predictive performance of our method. This supervised similarity distinguishes our approach from local and kernel regression~\citep{cleveland1988locally}, and is inspired by the success of row-wise attention in foundation models for tabular data, particularly TabPFN~\citep{hollmann2025accurate}. 

The idea of \emph{supervised attention weights} is quite general, and we show how to apply it to standard supervised learning models for tabular data, time series and spatial data.
Our intuition is that ``one-size-fits-all'' models do not always provide the best prediction for each data point. Rather, data can be heterogeneous in ways that are difficult to enumerate or anticipate. Attention weighting naturally accommodates a potentially continuous spectrum of latent subgroups in the data. By assigning attention weights rather than discrete labels, it adapts to both hard clustering (distinct groups) and soft clustering (overlapping or blended subgroup memberships). 

The contributions of this work are as follows:
\begin{enumerate}
    \item We present a simple overview of attention and in-context learning (Section~\ref{sec:background}), and its connection to local linear regression and kernel methods (Section~\ref{sec:other_related}).
    \item We develop a method to use a simplified version of attention to fit interpretable, personalized models for tabular data (Sections~\ref{sec:method} and~\ref{sec:generalization}) whose predictive performance typically matches or exceeds that of its non-attention counterparts. We focus on the application of this idea to the lasso, termed \textit{attention lasso}.
    \item We extend our method to spatial and time-series data (Section~\ref{sec:extensions}).
    \item We propose a way to interpret the personalized models (Section~\ref{sec:interpretability}). Rather than identify a single set of predictive features for a dataset, we obtain a coefficient vector for each test point and cluster these vectors to summarize heterogeneity. We also visualize the distribution of coefficient values across patients in the MSK-CHORD data (Section~\ref{sec:examples:msk}).
    \item We analyze a hard-attention version of attention lasso, showing that when the regression function has different linear models on different regions of the feature space, the local lasso has lower prediction risk than any global linear fit (Section~\ref{sec:theory}).
\end{enumerate}
We hope that this work will be useful to researchers interested in prediction and in characterizing heterogeneity, and to those seeking to understand how the idea of attention connects to local regression and kernel methods.

The rest of this paper is as follows. Section~\ref{sec:method} shows our method in detail as applied to lasso linear regression (\emph{attention lasso}), and is followed by real data examples in Section~\ref{sec:examples}. Section~\ref{sec:background}  gives a brief overview of attention and in-context learning, emphasizing the
connection to supervised weighting of training examples. Section~\ref{sec:other_related} discusses related work, including kernel regression and local linear regression. We extend our method to spatial and time series data in Section~\ref{sec:extensions}, and to other base learners in Section~\ref{sec:generalization}. Section~\ref{sec:theory} summarizes our theoretical guarantees.

We close with a discussion in Section~\ref{sec:discussion}. Appendix~\ref{app:examples:sims} shows results of a simulation study; Appendix~\ref{app:theory} analyzes a hard-attention variant of attention lasso under a piecewise affine regression model, giving support and sign recovery under a local-information condition, and a prediction risk that is strictly smaller than that of any global affine fit when the regression function is not affine. Appendix~\ref{sec:softmax_gaussian} shows the relationship between attention and Gaussian kernel regression.

\section{Supervised learning with attention for tabular data}
\label{sec:method}
\subsection{General procedure}
We first describe \textit{attention for supervised learning}
as applied to  tabular data. 
While we focus on continuous outcomes, our approach applies to any response type, including binary, survival, and multiclass.
The special case of the lasso is described in the next section.

Figure~\ref{fig:overview} and Algorithm~\ref{alg:attention_super} describe our method in detail. Given training data $\bm{X}$ with response $\bm{y}$ and prediction point $\bm{x}^*$, we fit a weighted model where the weights reflect similarity to $\bm{x}^*$. One natural similarity measure, inspired by neural network attention (Section~\ref{sec:background}), is
\begin{equation}
   \text{softmax}\left(\bm{x}^* \bm{\hat W} \bm{X}^T\right),
   \label{eqn:XXT}
\end{equation}
where $\bm{\hat W}$ is a diagonal matrix whose entries are the absolute values of coefficients from ridge regression fit to $(\bm{X}, \bm{y})$. This ensures that two points are considered similar only with respect to features predictive of $\bm{y}$. This formulation is much simpler than even a single linear attention head, which is typically a dense, non-symmetric matrix. However, the diagonal structure of $\bm{\hat W}$ cannot represent feature interactions.

As a critical enhancement, we instead use random forest proximity as our similarity measure: we fit a random forest to $\left(\bm{X}, \bm{y}\right)$ and define the proximity between two points as the proportion of trees in which they land in the same terminal node. All experiments in this paper use random forest proximity, which in preliminary experiments performed substantially better than the ridge-based weights of Equation~(\ref{eqn:XXT}).

Because we fit a separate model for each test observation, we employ a regularization strategy to control complexity: the final prediction is a convex combination of a \emph{baseline model} (fit without weights to the full dataset) and an \emph{attention model} (the weighted model):
\begin{equation}
\hat y^* = (1 - m) \hat y^*_\text{base} + m \hat y^*_\text{attn},
\end{equation}
where $m$ is a mixing hyperparameter selected via cross-validation. When $m = 0$, we recover a single global model; when $m=1$, predictions are entirely personalized.

\begin{figure}[H]
    \centering
    \includegraphics[width=\linewidth]{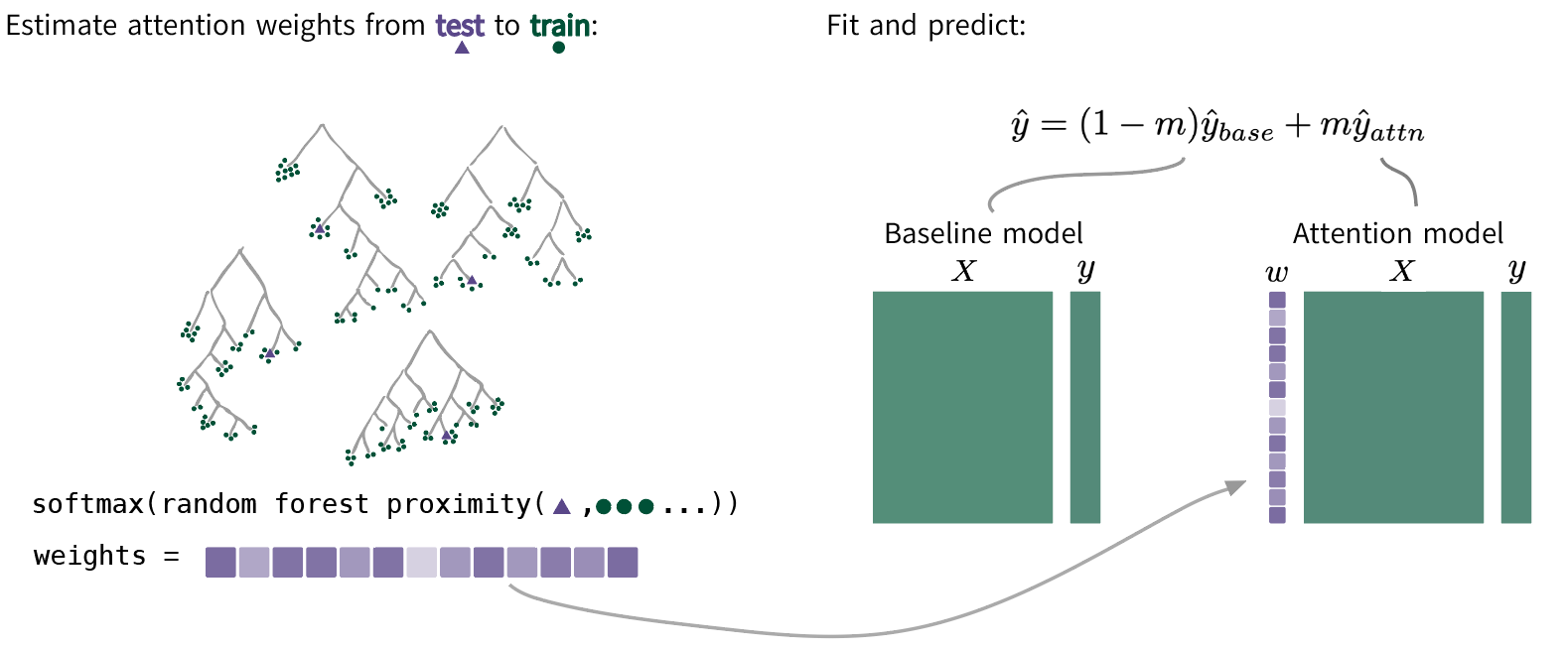}
    \caption{\em An outline of supervised learning with attention. First, we fit a random forest to estimate similarity between observations; second,  a  baseline  model (for example, lasso or boosting) is fit to the training data.  Then for each test observation $\textbf{x}^*$, we estimate attention weights using the random forest similarities between $\textbf{x}^*$ and each of the training observations;  we use these weights to fit an attention-weighted  model specific to $\textbf{x}^*$. Finally, we blend the weighted  model with the baseline model to make our prediction at $\textbf{x}^*$. More details are given in  Algorithm~\ref{alg:attention_super}.}
    \label{fig:overview}
\end{figure}

\begin{algorithm}[h]
    \caption{Supervised Learning with Attention } 
    \textbf{Input:} Train set $\bm{X}, \bm{y}$ with $n$ observations, test set $\bm{X}^*$, mixing parameter $m \in [0, 1]$. \\
    \textbf{Output:} Predictions $\bm{\hat{y}}^*$ 
    for each row of $\bm{X}^*$. 
    \hrule
    \begin{enumerate}
    \item \textbf{Compute attention weights $\bm{\hat{A}}^*$:}  Fit a random forest  to $(\bm{X}, \bm{y})$. For each observation $\bm{x}_i^*$ of $\bm{X}^*$, define $n$ similarity scores
    $\bm{\hat{S}}^*$ equal to the proportion of times an observation $\bm{x}_i^*$ falls into the same terminal node as each of the rows of $\bm{X}$.

    Define attention weights
        \begin{equation}    
            \bm{\hat{A}}^* = \text{softmax}(\bm{\hat{S}}^*),
        \end{equation}
    
    where \text{softmax} is applied row-wise to $\bm{\hat{S}}$. The $i^\text{th}$ row of the attention matrix $\bm{\hat{A}}^*$ corresponds to the attention scores from test observation $i$ to each of the train observations.
    
    \item \textbf{Fit baseline  model:} Fit a supervised learning model to $\bm{X}, \bm{y}$. Predict for $\bm{X}^*$ to obtain $\bm{\hat{y}}^*_{\text{base}}$.
    
    \item \textbf{Fit attention-weighted  models:} For $\bm{x}^*_i$, the $i^\text{th}$ row of $\bm{X}^*$, fit a weighted supervised learning model  to $\bm{X}, \bm{y}$ using weights $\bm{\hat{A}}^*_i$. Predict at $\bm{x}^*_i$ to obtain $\bm{\hat{y}}^*_{i, \text{attn}}$.
    
    \item \textbf{Combine predictions:} Return the weighted average
    \[
    \bm{\hat{y}}^* = (1 - m) \bm{\hat{y}}^*_{\text{base}} + m \bm{\hat{y}}^*_{\text{attn}}.
    \]
\end{enumerate}
\hskip 0.5in Note: the mixing parameter $m$ is selected  through cross-validation.
\label{alg:attention_super}
\end{algorithm}

\subsection{The attention lasso}
In much of this paper, we focus on the application of the above idea  to the special case of the lasso, detailed in this section. 
We briefly explore its application to other models---
especially gradient boosting---
in Section~\ref{sec:generalization}.

Algorithm \ref{alg:attention_lasso} shows the details of our proposal when the lasso is used for both the baseline and attention-weighted models.
\begin{algorithm}[h]
    \caption{ Attention Lasso } 
    \textbf{Input:} Train set $\bm{X}, \bm{y}$ with $n$ observations, test set $\bm{X}^*$, mixing parameter $m \in [0, 1]$. \\
    \textbf{Output:} Predictions $\bm{\hat{y}}^*$ for $\bm{X}^*$, 
    for each row of $\bm{X}^*$. 
    \hrule
    \begin{enumerate}
    \item \textbf{Compute attention weights $\bm{\hat{A}}^*$:}  Fit a random forest  to $(\bm{X}, \bm{y})$. For each observation $\bm{x}_i^*$ of $\bm{X}^*$, define $n$ similarity scores
    $\bm{\hat{S}}^*$ equal to the proportion of times an observation $\bm{x}_i^*$ falls into the same terminal node as each of the rows of $\bm{X}$.

    Define attention weights
        \begin{equation}    
            \bm{\hat{A}}^* = \text{softmax}(\bm{\hat{S}}^*),
        \end{equation}
    
    where \text{softmax} is applied row-wise to $\bm{\hat{S}}$. The $i^\text{th}$ row of the attention matrix $\bm{\hat{A}}^*$ corresponds to the attention scores from test observation $i$ to each of the train observations.
    
    \item \textbf{Fit baseline  model:} Fit a lasso regression to $(\bm{X}, \bm{y})$ and select regularization parameter $\hat{\lambda}$. Predict for $\bm{X}^*$ to obtain $\bm{\hat{y}}^*_{\text{base}}$.
    
    \item \textbf{Fit attention-weighted lasso models:} For $\bm{x}^*_i$, the $i^\text{th}$ row of $\bm{X}^*$, fit a weighted lasso regression to $\bm{X}, \bm{y}$ using weights $\bm{\hat{A}}^*_i$ and regularization parameter $\hat{\lambda}$. Predict at $\bm{x}^*_i$ to obtain $\bm{\hat{y}}^*_{i, \text{attn}}$. 
    
    \item \textbf{Combine predictions:} Return the weighted average
    \[
    \bm{\hat{y}}^* = (1 - m) \bm{\hat{y}}^*_{\text{base}} + m \bm{\hat{y}}^*_{\text{attn}}.
    \]
\end{enumerate}
\label{alg:attention_lasso}
\end{algorithm}
 To control complexity, all of the  attention models share a common lasso regularization parameter $\hat{\lambda}$, selected through cross-validation when fitting the baseline model.
 
We interpret fitted models by clustering coefficient vectors using \texttt{protoclust}, a hierarchical clustering method where each cluster is represented by a prototypical element~\citep{bien2011hierarchical, protoclust}. This groups test data according to their fitted models and enables interpretation. Importantly, because the lasso produces sparse coefficients, clustering occurs in a lower-dimensional space than the original features, which is typically much simpler computationally and easier to interpret. Example clusterings and visualizations appear in Section~\ref{sec:examples}.

\begin{remark}
\small 
{\em Computational cost.} 
Attention lasso fits a separate model for each test point, which may seem prohibitive. However, model fitting is embarrassingly parallel across test points. Additionally, the computational cost of attention lasso is not without precedent; it is similar to leave-one-out CV. For application to other models--- for example gradient boosting (Section~\ref{sec:generalization})--- other computational tricks are  needed.
\end{remark}

\begin{remark}
\small
{\em Cost of cross-validation.}
Selecting $m$ via cross-validation is efficient because it does not require refitting models for each candidate value. For each fold $k$, apply the full training algorithm using $(\bm{X}_{-k}, \bm{y}_{-k})$ as the training set and $\bm{X}_k$ as the prediction set. Store the baseline lasso predictions $\bm{\hat{y}}_{\text{base}, k}$ and the attention model predictions $\bm{\hat{y}}_{\text{attn}, k}$, and use these pre-computed predictions to compare each candidate mixing value (e.g. $m \in \{0, 0.1, 0.2, \ldots, 1\}$): ${\bm{\hat{y}}_k(m) = (1-m)\bm{\hat{y}}_{\text{base}, k} + m\bm{\hat{y}}_{\text{attn}, k}}$. Finally, estimate the cross-validation error for each candidate $m$ and select the value $\hat{m}$ that minimizes the mean error across folds.
\end{remark}
\begin{remark}
\small
{\em Adaptive selection of a mixing parameter for each test point.}
For each prediction point, attention lasso blends the standard lasso model with a weighted model. So far, we have selected a single mixing parameter $m$ for all data. But perhaps some data points are well-represented by the global model and therefore should have $m = 0$, while others are from a smaller subset within the data and would benefit from $m$ closer to $1$. To compute \emph{adaptive} mixing values for a given prediction point $\bm{x^*}$, compute a \emph{weighted} error metric using the attention weights for $\bm{x^*}$ when doing cross-validation. Because cross-validation predictions are already stored, this can be done in a vectorized fashion across all test points at once. By optimizing this weighted metric, we are selecting the mixing parameter $\hat m^*$ that gives the best performance for data similar to $\bm{x^*}$.
\end{remark}
\begin{remark}
\small
{\em Softmax temperature.}
If desired we may use a ``temperature'' hyperparameter to augment softmax: before applying softmax, simply divide the scores by the temperature. For temperature $<1$, this will cause scores to concentrate more heavily on a small subset of training examples; temperature $>1$ will make scores more spread out and diffuse.
\end{remark}

\section{Examples}
\label{sec:examples}

\subsection{Example: UC Irvine Machine Learning Repository}
\label{sec:examples:uci}

We compare attention lasso to the lasso, local linear forest (LLF), XGBoost, LightGBM, random forest, and K-nearest neighbors (KNN), using \texttt{glmnet}~\citep{glmnet}, \texttt{grf}~\citep{grf}, \texttt{lightgbm}~\citep{lightgbm}, \texttt{xgboost}~\citep{xgb}, \texttt{ranger}~\citep{ranger}, and \texttt{FNN}~\citep{FNN} at defaults except as noted. Random forest used 500 trees; XGBoost $\eta=0.1$, depth 6, rounds by CV with early stopping; LightGBM up to 500 rounds, unlimited depth; KNN chose $K \in \{3,5,10,15\}$ by CV. For LLF we used 500 trees and, per \texttt{grf}'s recommendation in higher dimensions, restricted the local linear correction to the lasso-selected variables, so LLF and attention lasso both localize by supervised selection. Attention lasso and the lasso tuned shrinkage (and, for attention lasso, the mixing parameter) by 10-fold CV. Throughout, we report \emph{relative improvement} in prediction squared error over the lasso, $100\times \frac{(\text{lasso PSE} - \text{model PSE})}{\text{lasso PSE}}$.

We consider 12 regression datasets from the UCI Machine Learning Repository~\citep{UCIRepository} with fewer than $n = 5000$ observations. For each dataset, we ran 50 iterations of a 50/50 train/test split and imputed missing values using train-set column means. A few datasets required preprocessing: we log-transformed the response $\log(y+1)$ for Facebook Metrics and averaged responses in multi-response datasets. Our results are summarized in Table~\ref{tab:results_22datasets}.

Attention lasso outperformed standard lasso on 11 of 12 datasets; in the one exception (Forest Fires), the difference was negligible (0.4\%), as the mixing parameter effectively reverted to the baseline when local adaptation provided no benefit. Looking across all methods, no single approach dominated: attention lasso had the best performance on 4 datasets; random forest, LightGBM and XGBoost each on 2; local linear forest, lasso, and KNN each on 1. However, attention lasso retains interpretability unlike the more complex methods (Section~\ref{sec:interpretability}).

\begin{table}[h]
\centering
\begin{footnotesize}
\setlength{\tabcolsep}{4pt}
\begin{tabular}{lrr|rrrrrr}
\toprule
Dataset & n & p & Attention & LLF & LightGBM & XGBoost & RF & KNN \\
\midrule
Airfoil Self-Noise & 1503 & 5 & 75.0 (1.1) & 70.4 (0.3) & \textbf{84.3 (0.3)} & 84.1 (0.3) & 69.3 (0.3) & 48.2 (0.8) \\
Auto MPG & 398 & 7 & \textbf{31.8 (0.9)} & 30.6 (0.9) & 26.1 (1.0) & 17.2 (1.1) & 26.0 (1.1) & 10.1 (1.2) \\
Automobile & 205 & 25 & \textbf{37.1 (1.3)} & 9.2 (2.0) & 27.4 (1.6) & 12.3 (3.1) & 32.6 (1.2) & -25.6 (2.5) \\
Communities \& Crime & 1994 & 127 & \textbf{3.1 (0.4)} & -1.7 (0.7) & -1.5 (0.4) & -13.5 (0.7) & 1.2 (0.3) & -19.0 (0.7) \\
Conc.~Comp.~Strength & 1030 & 8 & 62.8 (0.7) & 69.3 (0.3) & \textbf{77.7 (0.3)} & 74.4 (0.4) & 64.3 (0.4) & 17.6 (1.0) \\
Facebook Metrics & 500 & 18 & 93.6 (0.9) & 89.8 (0.5) & 90.5 (0.4) & \textbf{94.1 (0.5)} & 93.4 (0.3) & 56.8 (2.0) \\
Forest Fires & 517 & 12 & -0.4 (0.3) & -0.6 (0.5) & -0.1 (0.8) & -9.0 (4.1) & -21.1 (6.4) & -24.5 (5.6) \\
Infrared Therm.~Temp.~& 1020 & 33 & 3.4 (0.5) & \textbf{12.7 (0.7)} & -1.8 (1.0) & -4.3 (1.1) & 5.8 (0.7) & -11.2 (0.9) \\
Liver Disorders & 345 & 5 & 0.2 (0.7) & 1.8 (1.3) & 1.7 (1.0) & -12.5 (1.7) & \textbf{3.6 (1.0)} & \textbf{3.6 (1.2)} \\
Real Estate Valuation & 414 & 6 & 18.2 (1.3) & 22.6 (0.6) & 24.2 (0.8) & 14.0 (2.0) & \textbf{29.5 (0.9)} & 8.6 (0.8) \\
Servo & 167 & 4 & 63.8 (1.7) & 68.1 (1.5) & 23.3 (1.9) & \textbf{74.0 (1.6)} & 46.0 (1.5) & -21.9 (4.5) \\
Stock Portfolio Perf.~& 315 & 12 & \textbf{60.4 (0.8)} & 49.6 (1.1) & 32.1 (1.9) & 23.9 (2.4) & -15.9 (2.8) & -200.7 (6.1) \\
\bottomrule
\end{tabular}
\caption{\em Mean (SE) of relative improvement (\%) over lasso across 12 datasets in the UCI repository (higher is better). Best method per dataset in \textbf{bold}. Attention lasso has the best performance on 4 of 12 datasets; LightGBM, XGBoost and random forest on 2 each; and local linear forest, lasso, and KNN on 1 each. Attention lasso outperforms lasso on 11 of the 12, and nearly matches on the last (Forest Fires).}
\label{tab:results_22datasets}
\end{footnotesize}
\end{table}

\subsubsection{Interpretation}
\label{sec:interpretability}
A lasso regression model is easy to interpret: it consists of a single, sparse vector of coefficients for the entire dataset. Although attention lasso produces a coefficient vector for each test point, these can be summarized by clustering.

To interpret the fitted models from attention lasso, first compute the blended model coefficients for each data point: $(1 - m) \hat{\bm{\beta}}_{\text{base}} + m \hat{\bm{\beta}}_{\text{attn}}$. Then cluster these blended coefficients with \texttt{protoclust}~\citep{bien2011hierarchical} and visualize them with a heatmap to reveal within-cluster patterns. Figure~\ref{fig:interpretability_plots} illustrates the result of this pipeline. 

 \begin{figure}[H]      
 \centering        
      \begin{subfigure}{0.44\linewidth}   
      \centering    
          \includegraphics[width=\linewidth]{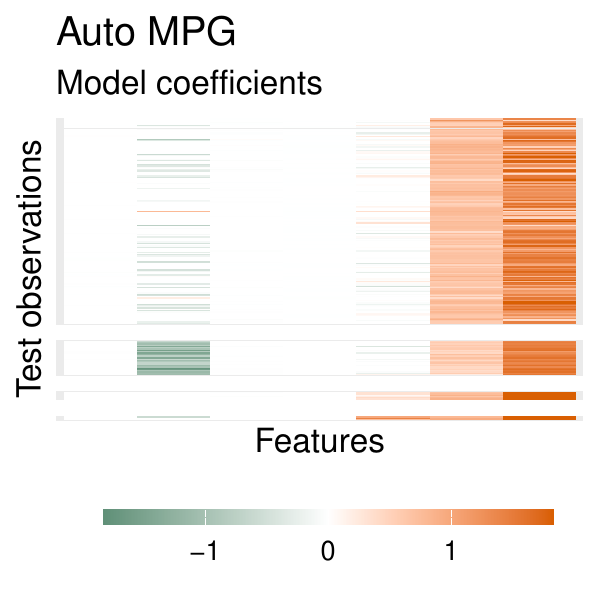}    
          \end{subfigure}   
          \hfill 
          \begin{subfigure}{0.45\linewidth}
          \centering          
          \includegraphics[width=\linewidth]{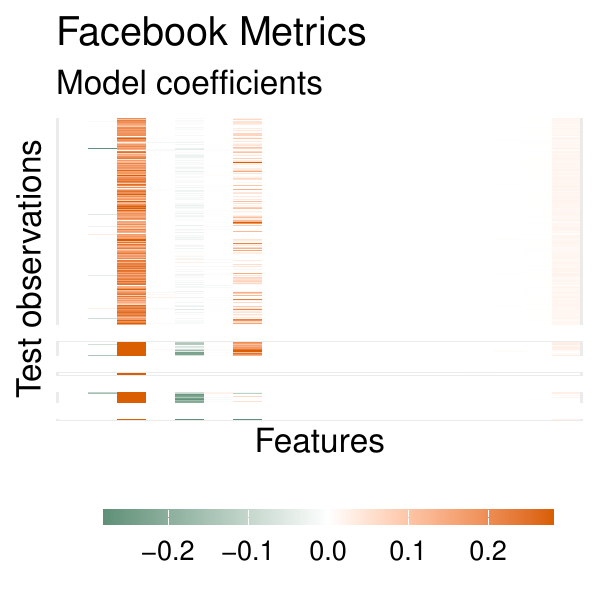}               
          \end{subfigure}   
      \caption{\em Coefficient heatmaps for two UCI datasets. Each row is the fitted model for a test observation. Clustered rows reveal patterns of heterogeneity in the coefficients.}                              \label{fig:interpretability_plots} 
      \end{figure}                                             

\subsection{Cancer survival prediction with MSK-CHORD}
\label{sec:examples:msk}
We analyzed data from the MSK-CHORD dataset~\citep{jee2024automated} with stage IV disease at diagnosis across five cancer types: breast (n=1,021), non-small cell lung cancer  (NSCLC; n=3,207), colorectal (CRC; n=2,140), prostate (n=1,143), and pancreatic (n=1,305). There are three levels at which one might study this dataset: (1) pancancer, (2) individual cancer and (3) individual patient (using attention weights). In (1), we pool all data, and therefore have the highest power to learn generalizable signal across cancer types; however, we average signal across cancers, and we are restricted to use features which are measured across all patients. In (2), we can use cancer-specific features and discover cancer-specific signal, but we average out patient-level heterogeneity. In (3), we can potentially make more precise patient predictions, but at the cost of increased variance. 

Here, we fit lasso-penalized Cox regression models for all three settings to predict overall survival. We fit a pancancer model, cancer-specific models, and attention-weighted patient-specific models. Then, we performed a grid search using out-of-fold predictions to find nonnegative coefficients to combine the three predictions in a weighted average.

Data were split 80/20 into training and test sets, stratified by death status. Our feature definitions followed the data processing from \citet{jee2024automated}. There were 40 features in common across all cancers, including demographics (age, sex, race), sequencing year, tumor mutational burden, metastatic sites at diagnosis (bone, brain, liver, lung, adrenal, lymph nodes, pleura, other), PDL1 status, and 23 cancer-associated genes (including  TP53, KRAS, EGFR, PTEN, PIK3CA, BRAF). Cancer-specific features included hormone receptor and HER2 status for breast cancer; smoking status and histologic subtype for NSCLC; MSI/dMMR status, tumor location, and CEA for CRC; PSA for prostate cancer; and CA19-9 for pancreatic cancer. In each cancer, binary features with fewer than 10 positive cases in the Stage IV subset were excluded. Overall survival was measured from diagnosis to death or last follow-up. Genomic data were typically available only after diagnosis, and we used left-truncation to account for immortal time bias. 

\begin{table}[htbp]                                                                                                                 
  \centering                                                  
  \begin{tabular}{llccccc}                                    
  \toprule                                                    
  Metric & Model & BRCA & NSCLC & CRC & Prostate & Pancreatic \\                                        \midrule                                            
  \multirow{4}{*}{C-index} 
  & Pancancer        & 0.609 & 0.556 & 0.544 & 0.705 & 0.528 \\               
  & Cancer-specific  & 0.706 & 0.631 & 0.605 & 0.738 & 0.612 \\  
  & Attention        & \textbf{0.722} & \textbf{0.639} & \textbf{0.623} & \textbf{0.760} & \textbf{0.619} \\                                   
  & Combined         & 0.716 & \textbf{0.639} & \textbf{0.623} & 0.756 & 0.618 \\              
  \midrule                     
  \multirow{4}{*}{Deviance} 
  & Pancancer        & 765.8 & 4750.1 & 2557.4 & 828.8 & 1771.1 \\            
  & Cancer-specific  & 730.1 & 4697.2 & 2530.5 & 804.1 & 1746.6 \\            
  & Attention        & 727.2 & \textbf{4683.4} & \textbf{2519.8} & \textbf{797.4} & \textbf{1738.8} \\                
  & Combined         & \textbf{727.0} & \textbf{4683.4} & \textbf{2519.8} & 805.6 & 1744.7 \\                                     
  \midrule                    
  \multirow{3}{*}{Blending weights} 
  & Pancancer        & \textbf{0.2} & 0.0 & 0.0 & \textbf{0.35} & 0.0 \\     
  & Cancer-specific  & \textbf{0.1} & 0.0 & 0.0 & 0.0 & \textbf{0.8} \\         
  & Attention        & \textbf{0.7} & \textbf{1.0} & \textbf{1.0} & \textbf{0.65} & \textbf{0.2} \\      
  \bottomrule                                          
  \end{tabular}               
  \caption{\em Test set performance for Stage IV patients in the MSK-Chord data. We fit lasso-penalized Cox models: pancancer (one model fit to all data), cancer-specific (one model for each cancer including cancer-specific    
  features), attention (one model for each test observation using attention weights) and combined (using grid search to blend the     
  three preceding models). Best performing models in bold. Attention models had best performance for most cancer types.}          
  \label{tab:msk_chord}                    
  \end{table}      

  \begin{figure}[H]   
  \centering                                 \includegraphics[width=0.9 \linewidth]{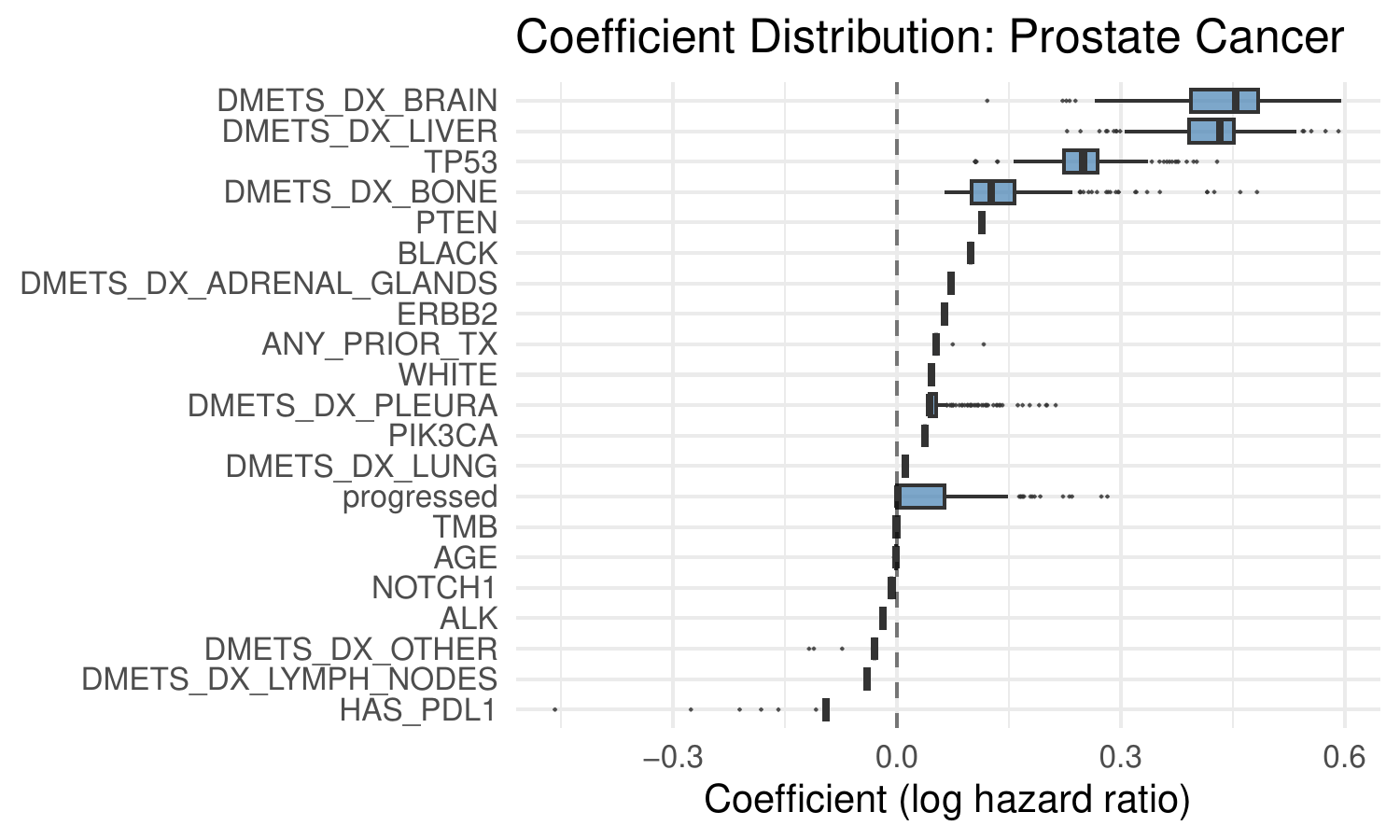}\caption{\em Distribution of coefficients across 228 prostate cancer patients in the MSK-CHORD dataset. Some coefficients are relatively constant across patients (e.g PTEN expression), while others have more variability (e.g. brain metastasis). Coefficients near 0 for more than half of patients were omitted for visual clarity.}  
  \label{fig:prostate-coefs}
  \end{figure}   

\section{Background: attention and in-context learning}
\label{sec:background}

This work is motivated by attention and in-context learning, two ideas that have played a central role in modern foundation models \citep{vaswani2017attention,dong-etal-2024-survey}. 
We do not use a transformer model in this paper.
Rather, we borrow the central statistical idea: when making a prediction for a new point, one can place more weight on training examples that are relevant to that point.

In a transformer, attention constructs weights between rows of an input matrix. 
If \(\bm{X}\) denotes a matrix of embedded tokens, a standard self-attention layer has the form
\[
    \operatorname{Attention}(\bm{Q},\bm{K},\bm{V})
    =
    \operatorname{softmax}\left( \frac{\bm{Q}\bm{K}^T}{\sqrt{d_k}} \right)\bm{V},
\]
where \(\bm{Q}=\bm{X}\bm{W}_q\), \(\bm{K}=\bm{X}\bm{W}_k\), and \(\bm{V}=\bm{X}\bm{W}_v\) are the query, key, and value matrices, and \(d_k\) is the key dimension \citep{vaswani2017attention}.
Thus, apart from the value transformation \(\bm{V}\), the key object is a row-normalized similarity matrix. 
Each row gives a set of weights describing how much one token attends to the others. 
In this sense, attention produces a data-adaptive weighted average, with the weights learned from the data rather than fixed in advance.

In-context learning uses this same mechanism to adapt predictions to the examples provided in the prompt. 
A model is given several demonstration pairs and a new query point, and it uses attention to relate the query to the demonstrations. 
Although the model parameters are not updated at prediction time, the prediction is nevertheless personalized to the
examples in the context. 
Recent work has shown that, in some settings, transformers can use in-context learning to emulate simple optimization procedures such as gradient descent \citep{von2023transformers,ren2024towards,deutch2023context}.

Our method translates this idea into a simpler supervised-learning framework. 
For a test point \(\bm{x}^*\), we construct weights over the training observations,
\[
    w^*(\bm{x}_i) \propto \exp\{ S(\bm{x}^*,\bm{x}_i) \},
\]
where \(S(\bm{x}^*,\bm{x}_i)\) is a supervised similarity score. These weights are then used to fit a local, weighted model for \(\bm{x}^*\). 
In the simplest analogy to neural attention, one might take
\[
    S(\bm{x}^*,\bm{x}_i) = \bm{x}^{*T} \bm{W} \bm{x}_i,
\]
where \(\bm{W}\) reflects the predictive importance of the features. 

Attention and in-context learning suggest that prediction for \(\bm{x}^*\) should depend most strongly on training examples that are relevant to \(\bm{x}^*\), where relevance is learned in a supervised way. 
Their success in prediction led us to revisit local regression, which is similar in spirit, but typically defines the weights from the covariates alone, for example through a distance to \(\bm{x}^*\). 
In this work we replace those weights with supervised weights from random forest proximity: two observations are similar if they often fall in the same terminal nodes of a forest trained to predict \(\bm{y}\). 
This lets the similarity depend on the nonlinearities and interactions that are useful for predicting the response. 
Attention lasso then uses these weights to fit an interpretable, regularized local model.
Because each local fit uses fewer effective observations than the full training set, local regression has higher variance; we limit this by blending the local fits with a single global fit.

\section{Related work}
\label{sec:other_related}

\subsection{Kernel methods and local regression}
Our attention approach is closely related to kernel-based methods that weight training observations by their similarity to a test point. We compare the softmax weighting scheme used in the attention mechanism with the model fitting approaches of kernel regression and locally weighted regression.

\subsubsection{Nadaraya--Watson kernel regression and attention}
The Nadaraya--Watson (N--W) kernel estimator~\citep{nadaraya1964estimating,watson1964smooth} with a Gaussian kernel defines the weight between a query point $\bm{x}^*$ and training point $\bm{x}_i$ as:
\[
  \frac{\exp\left(-\frac{\|\bm{x}^* - \bm{x}_i\|^2}{2\sigma^2}\right)}{\sum_{j=1}^n \exp\left(-\frac{\|\bm{x}^* - \bm{x}_j\|^2}{2\sigma^2}\right)},
\]
where $\sigma$ is a bandwidth parameter. Attention weights take an analogous softmax form:
\[
 \frac{\exp(\bm{x}^{*T} \bm{W} \bm{x}_i)}{\sum_{j=1}^n \exp(\bm{x}^{*T} \bm{W} \bm{x}_j)},
\]
where $\bm{W}$ is a learned weight matrix. Importantly, the two approaches define similarity quite differently: the attention mechanism learns $\bm{W}$ using the response $\bm{y}$, while N-W regression defines similarity based only on geometric distance in covariate space. Additionally, N-W regression uses these weights to compute a weighted average of responses, $\hat{y}^* = \sum_i w_i y_i$, rather than performing continued model fitting (like neural networks and attention lasso).

\subsubsection{Locally weighted regression}

Locally weighted regression (LWR), including methods like LOESS~\citep{cleveland1988locally}, is structurally similar to attention lasso. LWR fits a separate weighted linear (or polynomial) model at each prediction point:
\[
\hat{\bm{\beta}}^* = \argmin_{\beta} \sum_{i=1}^n w_i (y_i - \bm{x}_i^T \bm{\beta})^2,
\]
where the weights $w_i$ typically decrease with Euclidean distance from $\bm{x}^*$. 

Drawing on the principle of supervised similarity from neural attention,  our method extends locally weighted regression by defining similarity in a \emph{supervised} manner, ensuring that two observations are considered similar when they are close with respect to features and feature interactions that relate $\bm{X}$ to $\bm{y}$. Our ridge-based attention uses feature importance to weight distances, while random forest proximity finds nonlinear relationships that Euclidean distance cannot represent. 

Additionally, attention lasso also incorporates sparsity through $L_1$ regularization, fitting models of the form:
\[
\hat{\bm{\beta}}^* = \argmin_{\beta} \sum_{i=1}^n w_i (y_i - \bm{x}_i^T \bm{\beta})^2 + \lambda \|\bm{\beta}\|_1,
\]
using a shared penalty parameter $\lambda$ across all test points to encourage consistent model complexity. Finally, our approach blends the global (baseline) model with the local (attention) model using a mixing parameter $m$ chosen through cross-validation, allowing the data to reveal the extent to which local adaptation improves over the single global model.

\subsection{Supervised neighborhoods from random forests}
\label{sec:related:supervised-neighborhoods}
Random forest proximity has been proposed as a supervised similarity measure. \citet{LinJeon2006} showed that a random forest can be viewed as an adaptive nearest-neighbor method: two points are ``near'' when they often share a terminal node, which depends on the response $\bm{y}$. \citet{BloniarzEtAl2016} build on this idea with their method Silo, which fits an unpenalized weighted linear regression at each prediction point using forest weights, and prove its consistency for continuous regression functions when the splits are grown on an independent sample and the leaves shrink. \citet{FriedbergEtAl2021} use forest weights for local linear fitting with a ridge penalty on the local coefficients (local linear forests).

Attention lasso differs in using a sparse local model with a penalty shared across test points, which yields interpretable coefficient vectors, and in blending the local model with a global model by cross-validation. Our theory is a same-sample analysis of a hard-attention variant: in a sparse, continuous, piecewise affine model, neighborhood purity is derived from a sufficient impurity decrease condition instead of assumed, and the local lasso recovers the support and signs of the local models and has lower prediction risk than any global affine fit when the regression function is not affine (Section~\ref{sec:theory} and Appendix~\ref{app:theory}).

\subsection{Customized training}
Customized training~\citep{powers2016customized} addresses data heterogeneity by first partitioning test observations into clusters and then fitting separate models within each cluster using nearby training points. This approach is most appropriate when subgroups are clearly defined and the chosen clustering aligns with the underlying data structure. However, customized training has limitations. First, it requires pre-specifying the number of clusters and making hard assignments of observations to groups. Second, because clustering is performed on test data, it cannot leverage information from the response $\bm{y}$ to define clusters meaningful for predicting $\bm{y}$. 
In contrast, our method adaptively determines the appropriate training set for each prediction point through a soft weighting mechanism that is explicitly informed by the relationship between features and the response.

\subsection{Other related work}
Other related prior works include \citet{Wang2018FLWPR} who develop a procedure which can deal with  heterogeneous populations that can be naturally grouped into several ordered subtypes, and \citet{Gao2021} who propose a matching scheme for the assessment of heterogeneous treatment effects.

\section{Attention for time series and spatial data}
\label{sec:extensions}

Thus far, we have defined similarity between a test and train point using random forest proximity on their covariates and the response. Now, we consider data where there are relationships between observations; for example, time series data consists of measurements of the same variables at different points in time, and likewise for spatial data across 2- or 3-dimensional coordinates. This is somewhat similar to \emph{retrieval} for neural networks. For simplicity, we first discuss the time series case.

With time series data, we define similarity not only with regards to the current covariates, but also to their context. For example, a feature value may be interpreted differently depending on whether its recent history has been increasing or decreasing. To address this, we fit a random forest using the current features as well as their lagged values. Our definition of ``similarity'' therefore prioritizes observations with similar values \emph{and similar contexts}. During model fitting, we then have the choice to include only current values or current and lagged values, whichever are more appropriate for the task. The same principle applies to spatial data. To define similarity across pixels in images, we fit our random forest using features from each individual pixel and its nearest neighbors. 

There are many possible modifications to the similarity definition:
\begin{itemize}
\item \textbf{Spatial symmetry:} In image data, the relative position of neighboring pixels likely does not matter—only that two pixels are adjacent. A simple approach is to use the mean and standard deviation of neighbors rather than their raw feature values.
\item \textbf{Categorical context:} In single-cell spatial data, we may define similarity using a cell's measured covariates together with neighboring cell \emph{types} rather than their features.
\item \textbf{Distance weighting:} For time series, we choose a lag; for spatial data, we choose a number of neighbors. These could be extended to a kernel weighting that downweights observations further away in time or space.
\end{itemize}

\subsection{Example: time series regression}
We consider a time series dataset $\bm{X}, \bm{y}$ where each row of $\bm{X}$ and the corresponding entry of $\bm{y}$ represent the same variables recorded across time. For instance, in finance, $\bm{y}$ might denote the return at time $t$ and $\bm{X}$ includes variables representing market indicators from previous periods. 

Time series forecasting typically faces the challenge of limited observations similar to the current moment. We leverage the attention mechanism to identify historically relevant training examples, both in terms of current features (measured at time $t$) and recent history (lagged indicators from time $< t$ included as columns of $\bm{X}$). 

We applied attention lasso to the US economic change dataset (\texttt{us\_change}) from the fpp3~\citep{fpp3} package in R, which contains quarterly percentage changes in personal consumption expenditure, disposable income, production, savings, and unemployment. This dataset spans 1972-2019 with quarterly measurements, totaling $n = 188$ rows. For baseline lasso, random forest, local linear forest, XGBoost, LightGBM and K-nearest neighbors, we included 10 lagged values of each variable, resulting in $p = 54$ features. For attention lasso, we use the lagged features only to compute random forest similarity; the baseline and attention models are fitted with the original 4 features only, and random forest proximity uses the most recent 1 year (lag = 4) to define similarity. The response variable is the percentage change in consumption.

We vary the training set size from 50\% to 90\% of the data and test on the remaining observations. Hyperparameters were chosen with 5-fold cross-validation split by time: for each fold $k \geq 2$, we train on all observations from folds 1 through $k-1$ and test on fold $k$.

In this example, attention lasso performed better with longer training history; with smaller datasets, tree ensembles with lagged features had the best performance. 


\begin{table}[ht]
\centering
\begin{tabular}{rrrrrrrr}
  \toprule
Train \% & Attention & RF & LLF & XGBoost & LightGBM & KNN & Mixing \\ 
  \midrule
50 & 21.1 & 64.4 & -4.4 & 63.1 &\textbf{ 64.7 }& 61.0 & 0.7 \\ 
  60 & -14.9 & \textbf{48.5 }& -46.8 & 38.5 & 45.8 & 36.7 & 0.8 \\ 
  70 & \textbf{4.6} & -72.2 & -4.3 & -54.6 & -61.3 & -130.5 & 0.2 \\ 
  80 & 15.1 & -31.5 & \textbf{37.3 }& -40.2 & -26.8 & -153.9 & 0.5 \\ 
  90 & \textbf{53.7} & -44.5 & 45.8 & -50.4 & -15.3 & -144.2 & 0.6 \\ 
  Mean & \textbf{15.9} & -7.1 & 5.5 & -8.7 & 1.4 & -66.2 & -- \\
   \bottomrule
\end{tabular}
\caption{\em Relative improvement (\%) over lasso baseline with the time series dataset \texttt{us\_change} from the fpp3~\citep{fpp3} package. Positive values indicate better performance than lasso, negative indicate worse. With limited history (50--60\% training) tree ensembles and KNN perform best. With longer history (80--90\% training) only the supervised-locality methods, attention lasso and the local linear forest, improve on the lasso; at 70\% only attention lasso does. In both cases the tree ensembles and KNN degrade sharply. Mixing is the cross-validated weight on the attention model (0 = baseline lasso, 1 = attention lasso).} 
\label{tab:timeseries_improvement}
\end{table}

\subsection{Example: spatial data}
Now we turn to a dataset of 45 DESI mass spectrometry images of prostate  tissues, taken from \citet{Banerjee}. Each pixel in a mass spectrometry image represents a location at which 1{,}600 molecular abundances were measured. We aim to label the pixels in each image using these 1{,}600 features. Of the 45 images, 17 correspond to tumor tissue (pixel labels $y = 1$), and 28 to normal tissue (pixel labels $y = 0$). We have 17{,}735 pixels total.

We compared lasso and attention lasso across ten train/test splits, with entire images assigned to either train or test. For attention weights, we defined pixel similarity by training the random forest on each pixel's features together with those of its 8 spatial neighbors. This provides local context analogous to the temporal context in time series.

Across the ten splits, attention lasso outperformed lasso: it had an average AUC of 0.65 relative to lasso's average of 0.59. In 3 splits, lasso only fit the null model, presumably due to pixel heterogeneity masking global signal. In 2 of these cases, attention lasso was still able to find signal in local models. Performance is in Table~\ref{tab:spatial_results}, and an example is in Figure~\ref{fig:masspect}.

 \begin{table}[h]
  \centering
  \begin{tabular}{rccc}
  \toprule
  Split & Lasso & Attention lasso & Mixing (0: base, 1: attn) \\
  \midrule
  1 & 0.500 & \textbf{0.718} & 1.0 \\
  2 & 0.731 & \textbf{0.746} & 0.2 \\
  3 & 0.425 & \textbf{0.475} & 1.0 \\
  4 & 0.669 & \textbf{0.675} & 0.2 \\
  5 & \textbf{0.680} & 0.678 & 1.0 \\
  6 & 0.500 & \textbf{0.750} & 0.2 \\
  7 & \textbf{0.543} & \textbf{0.543} & 0.0 \\
  8 & \textbf{0.500} & \textbf{0.500} & 0.0 \\
  9 & \textbf{0.685} & \textbf{0.685} & 0.0 \\
  10 & \textbf{0.698} & 0.696 & 0.3 \\
  \midrule
  Mean & 0.593 & \textbf{0.647} & 0.390 \\
  Std. err &  0.035  &  0.032 & 0.137 \\
  \bottomrule
  \end{tabular}
  \caption{\em AUC comparison across train/test splits for DESI mass spectrometry data, distinguishing prostate cancer tissue from normal.}
  \label{tab:spatial_results}
  \end{table}

\begin{figure}[H]
    \centering
    \begin{subfigure}{0.43\textwidth}
      \includegraphics[width=\textwidth]{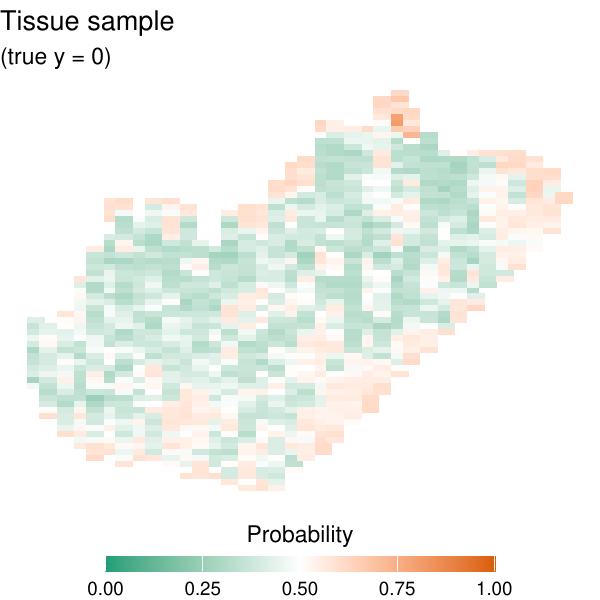}
    \end{subfigure}
    \hfill
    \begin{subfigure}{0.43\textwidth}
      \includegraphics[width=\textwidth]{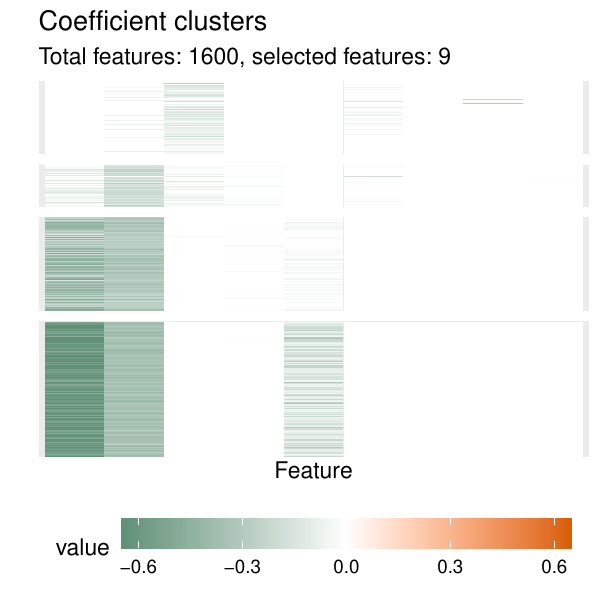}
    \end{subfigure}
    \caption{\em Example mass spectrometry image from split 1 with attention lasso mixing = 1 (chosen by cross-validation). Left: normal tissue sample with pixels colored by attention lasso model probability. Right: fitted model coefficients for the sample image.}
    \label{fig:masspect}
  \end{figure}

\section{Attention for other machine learning methods}
\label{sec:generalization}

So far we have focused on the special case of the lasso, but the algorithm we have described applies to supervised learning methods beyond the lasso. The steps are: (1) fit a random forest to estimate attention weights, (2) fit a base learner, (3) fit a \emph{weighted} base learner for each test point and (4) blend the predictions of the models from (2) and (3). The base learner could be, for example, XGBoost or LightGBM.

Nonlinear models like LightGBM typically adapt well to data heterogeneity: they can fit localized models in a way that linear models cannot. However, attention can still be useful. The attention weights focus each test point's model on the most relevant training examples; as a result, the localized models may not need to be as complex. In the case of gradient boosted trees, the attention models may be composed of fewer, shallower trees. This results in a more interpretable ensemble of models: first, the attention weights are useful for understanding similarity between test and train samples, and second, clustered feature importances across trees may characterize data heterogeneity (e.g. Figure~\ref{fig:interpretability_plots_lgbm}).

\begin{figure}[H]
    \centering
    \begin{subfigure}{0.44\linewidth}
        \centering
        \includegraphics[width=\linewidth]{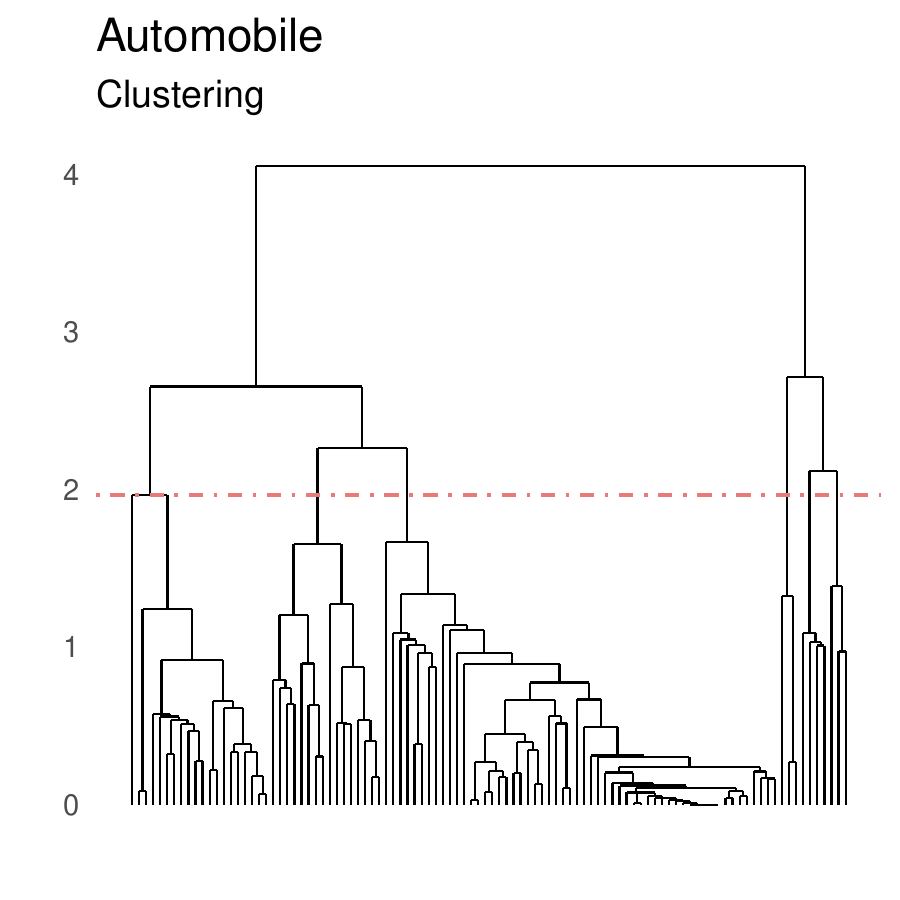}
    \end{subfigure}
    \hfill
    \begin{subfigure}{0.44\linewidth}
        \centering
        \includegraphics[width=\linewidth]{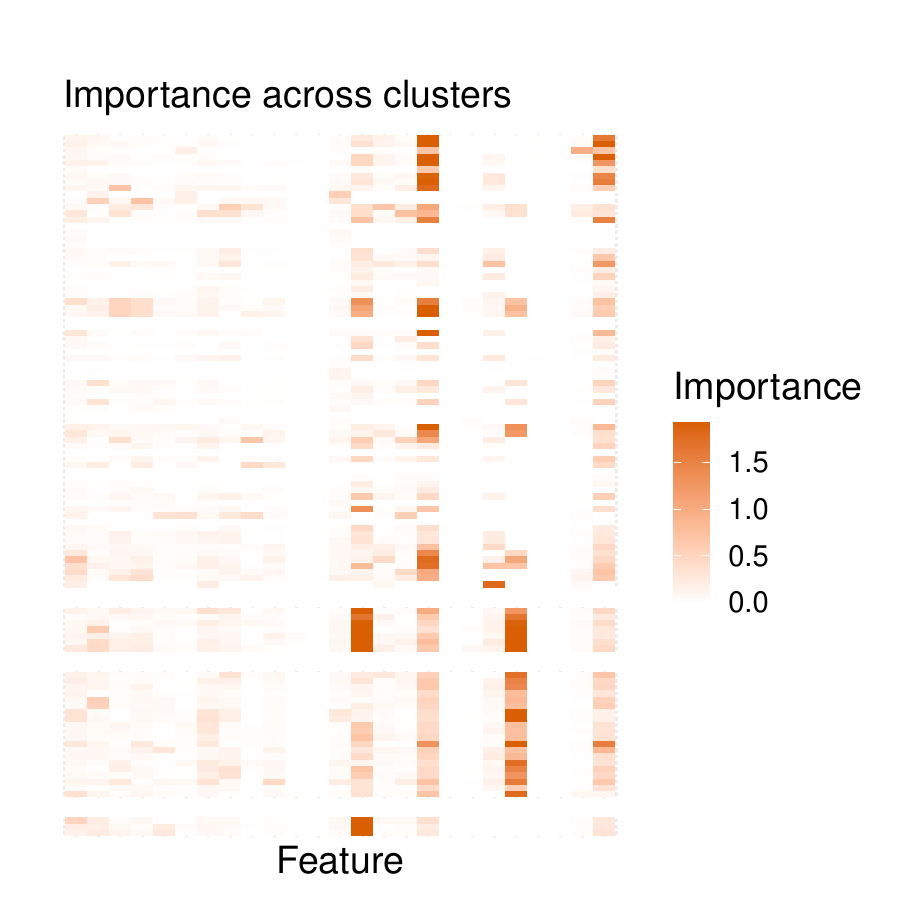}
    \end{subfigure}
    
    \caption{\em Clustered feature importances from Attention LightGBM for the Automobile dataset. Models were trained using 50\% of data. On held-out data, it improved prediction squared error over the lasso by 33.7\% (SE 1.3), compared to 27.4\% (SE 1.6) for standard LightGBM across 50 train/test splits. Each row is a fitted model for a test observation; clustering reveals heterogeneity in feature importances.}
    \label{fig:interpretability_plots_lgbm}
\end{figure}

\section{Theory}
\label{sec:theory}
In this section
we analyze a hard-attention variant of attention lasso that uses a just a single CART tree (analysis of a full random forest seems intractable due to its complexity).

Given training data $(\bm{X},\bm{y})$ with rows $(\bm{x}_i,y_i)_{i=1}^n$ and a test point $\bm{x}^*\in\mathbb R^p$:
\begin{enumerate}
  \item Grow a CART tree on all $n$ observations by greedy squared-error splits, splitting a node only if it contains more than $k$ points.
  \item Let $A$ be the smallest tree node containing $\bm{x}^*$ that holds at least $\ell$ training points. The neighborhood of $\bm{x}^*$ is every training observation in $A$, each with weight $1/|A|$.
  \item Fit a lasso using the observations in $A$ with an unpenalized intercept and features centered and scaled by their range over $A\cup\{\bm{x}^*\}$, and predict at $\bm{x}^*$ with the fitted model. The penalty is $\tau_A = K v \sqrt{a_n/|A|}$, where $a_n = p\log(2en)+3\log n$ is a complexity term of order $p\log n$, $K$ is a fixed large constant, and $v$ is the sub-Gaussian proxy for the noise defined below. Write $\widehat f(\bm{x}^*)$ for the prediction and $\widehat{\bm\beta}(\bm{x}^*)$ for the fitted slopes, converted back to the original feature units.
\end{enumerate}
The same observations are used to grow the tree and to fit the local lasso.

\paragraph{Data model.}
The training observations $(\bm{x}_i,y_i)$ are i.i.d. with $y_i=f(\bm{x}_i)+\xi_i$, where, conditionally on $\bm{x}_i$, the noise $\xi_i$ has mean zero and variance $\sigma^2$, and $\E[\exp(\theta_\xi\xi_i^2/\sigma^2)\mid\bm{x}_i]\le K_\xi$ for fixed constants $\theta_\xi>0$ and $K_\xi\ge1$ that do not depend on $n$. This implies that $\xi_i$ is sub-Gaussian with proxy $v$, that is, $\E[e^{t\xi_i}\mid\bm{x}_i]\le e^{v^2t^2/2}$ for all $t$, with $v^2=8K_\xi\sigma^2/\theta_\xi$ (Lemma~\ref{att:lem:noise-mgf}); the penalty may use any $v$ satisfying this inequality. The test point $\bm{x}^*$ is an independent draw from the same distribution. We assume:
\begin{enumerate}[label=(A\arabic*), ref=A\arabic*]
  \item \label{ass:design} \emph{Design.} The features of $\bm{x}_i$ are mutually independent. Those in an active set $S$, $|S|=s$, are uniform on intervals $[a_j,b_j]$ of length $L_j$; the others have arbitrary distributions.
  \item \label{ass:f} \emph{Signal.} $f$ depends only on $\bm{x}_S$, is continuous, and is piecewise affine on a finite partition into pieces, each a finite union of axis-aligned rectangles: on piece $g$, $f(\bm{x})=a_g+\sum_{j\in S}\beta_{gj}x_j$, where $\lambda_j\le|\beta_{gj}|\le R\lambda_j$ for some $\lambda_j>0$ and $R\ge1$.
\end{enumerate}
All pieces thus share the support $S$, though slopes and signs may differ. 
Continuity (no jumps between pieces) implies that $f$ has no interactions: $f(\bm{x})=c_0+\sum_{j\in S}g_j(x_j)$ with each $g_j$ continuous and piecewise affine (Proposition~\ref{att:prop:geometry}). 
Extending every side of every rectangle across the box gives a finite grid; let $J$ be the largest number of grid intervals in any one active feature (for $|x_1-\tfrac12|$, $J=2$).
These assumptions are chosen to suit CART: they imply the sufficient impurity decrease condition \citep{ChiEtAl2022,MazumderWang2023}, so the variation of $f$ within a node decays geometrically with depth and deep nodes lie within a single piece (Theorem~\ref{att:thm:purity}).

\begin{theorem}[Summary of Theorems~\ref{att:thm:purity},
\ref{att:thm:local}, and \ref{att:thm:risk}]
\label{thm:main}
Here $C$ and $C_0$ denote constants that do not depend on $n$, $p$, or the distributions of the inactive features. Let $g(\bm{x}^*)$ denote the piece containing $\bm{x}^*$ and $\lambda_*=\min_{j\in S}\lambda_jL_j$. Suppose $s$, $R$, and $J$ are bounded, and
$(\Var[f(\bm{x}_i)]+\sigma^2)/\lambda_*^2$ is bounded. As $n\to\infty$, suppose $p\log n=o(n)$, $k=o(n)$, $\ell=o(n)$, and $\ell\ge C_0s^2a_n$, and let $K$ be large enough. Then with probability tending to one:
\begin{enumerate}
  \item every training observation in the neighborhood of $\bm{x}^*$ lies in the piece $g(\bm{x}^*)$;
  \item $\supp\widehat{\bm\beta}(\bm{x}^*)\subseteq S$ and $|\widehat f(\bm{x}^*)-f(\bm{x}^*)|\le C(s+1)v\sqrt{a_n/\ell}$;
  \item if $\lambda_*\,\ell/n>Cv\sqrt{a_n/\ell}$, then $\supp\widehat{\bm\beta}(\bm{x}^*)=S$ and $\sgn\widehat\beta_j(\bm{x}^*)=\sgn\beta_{g(\bm{x}^*),j}$ for all $j\in S$.
\end{enumerate}
If also $|f|\le M$, the clipped predictor $\widehat f^M=\max\{-M,\min\{M,\widehat f\}\}$ satisfies $\E_{\bm{x}^*}\bigl[\{\widehat f^M(\bm{x}^*)-f(\bm{x}^*)\}^2\bigr]\to0$ in probability over the training sample.
\end{theorem}

\begin{proposition}[Comparison with global affine fits; Proposition~\ref{att:prop:global-gap}]
\label{prop:gap}
Let $\Delta_f=\inf_{a,\bm b}\E\bigl[\{f(\bm{x}_i)-a-\bm b^\top\bm{x}_{i,S}\}^2\bigr]$.
Every global affine predictor, however tuned, has risk at least $\sigma^2+\Delta_f$, and $\Delta_f>0$ unless $f$ is affine. 
If $f$ is not affine and fixed, then under the conditions of Theorem~\ref{thm:main} the hard-attention risk tends to $\sigma^2$, so it is strictly smaller than the risk of any global affine fit with probability tending to one.
\end{proposition}

For $y=|x_1-1/2|+\xi$ the two slopes cancel: the global lasso with a suitably large penalty sets all slopes to zero with probability tending to one, while the local procedure recovers the support $\{1\}$ and the slope sign (Example~\ref{att:ex:vee}).
These results isolate the statistical mechanism underlying attention lasso: a supervised, tree-based neighborhood localizes each query to its own piece, after which a local lasso recovers that piece's support and signs. 
A global affine fit cannot recover these, since it has a single slope per feature, while the slopes differ across pieces.
The analysis is for the hard-attention estimator above under a product design; extending it to forest-proximity softmax weights, cross-validated tuning, and dependent designs is left to future work. Proofs are in Appendix~\ref{app:theory}.

\section{Discussion}
\label{sec:discussion}
This paper adapts the attention mechanism from neural networks to a general framework for supervised learning on tabular data. By weighting training observations according to their \textit{supervised} similarity to each test point, our method fits models for heterogeneous data without requiring pre-specification of cluster structure. We also give theory for a hard-attention variant that fits the lasso in the nodes of a single CART tree. In a sparse, continuous, piecewise affine data generating model, the local lasso recovers the support and coefficient signs of each query's local linear model and predicts at the noise level, with probability tending to one, whereas, when the regression function is not affine, every global affine fit has risk bounded away from the noise level (Section~\ref{sec:theory} and Appendix~\ref{app:theory}). Empirical examples show that clustering the fitted coefficients reveals interpretable patterns that characterize data heterogeneity.

In self-attention, the attention weights $\bm W$ are learned end-to-end: they are optimized jointly with all other model parameters via backpropagation. In contrast, our approach first computes the weights through random forest proximity, and then separately fits a weighted model. We could instead have set up an optimization problem to jointly learn a weight matrix $\hat{\bm{W}}$ \emph{and} fit the prediction model. However, this non-convex objective is challenging to optimize in practice, and the dependence between the attention weights and the downstream model introduces additional complexity and the risk of poor local minima. For this reason, we instead prefer a two-step approach which can be implemented easily with standard, well-tested software.

We use random forest proximity to define the training weights, but note that any supervised weighting scheme that relates test points $\bm{x}^*$ to the training set $\bm{X}$ may be used. Random forest proximity is appealing because it naturally captures complex, nonlinear relationships that are useful for predicting $y$, and random forests are usually fast and straightforward to fit.

We expect there are other interesting applications or extensions of our approach. For example, to estimate conditional average treatment effects, we could design a method based on the \emph{R-learner} (and R-lasso)~\citep{nie2021quasi}. The R-learner orthogonalizes treatment effect estimation by first regressing out nuisance functions for the average treatment effect and propensity scores, then fitting a penalized regression for the heterogeneous treatment effect. We could use the attention mechanism to model the heterogeneous treatment effect to localize or personalize the estimation. Another idea is to use different model classes for the baseline (non-attention) model and the local (attention-weighted) models: the baseline model could be, for example, XGBoost, while the attention-weighted models are local lasso models, and the blending parameter $m$ is estimated via cross-validation. This may have improved predictive power; the baseline model is able to use the full dataset to flexibly learn interactions, while the attention-weighted models would adapt predictions locally. 

Our approach has limitations. For large datasets, fitting a separate model for every test point may seem computationally prohibitive. However, each local model is simply a weighted version of the base learner and is typically fast to fit. Moreover, because the test models are independent, they can be fit in parallel across CPU nodes. Further, the computational cost is on par with leave-one-out cross-validation, where a separate model is fit for each \emph{train} observation. When the base learner is a tree-based method, it may be possible to use attention-weighted averages from each leaf, rather than refitting the tree for each test observation. We also note that interpretability can become challenging when many subgroups are discovered, and we have not developed formal inference procedures (confidence intervals or hypothesis tests for the fitted coefficients), which is an important direction for future work. Similarly, our theoretical guarantees cover a hard-attention variant under a product design; extending them to the softmax forest-proximity weights, cross-validated tuning, and dependent designs would require separate control of the corresponding localization, design, and noise errors.
\medskip

{\bf Acknowledgements}. We would like to thank Ryan Tibshirani, Aaron Niskin, Isaac Mao, Caroline Kimmel, Trevor Hastie and Louis Abraham for helpful comments. The authors used Claude Opus 4.5 (model ID: claude-opus-4-5-20251101) and ChatGPT 5.0 for coding support and proof verification.

{\bf Funding}. R.T. was supported by the NIH (5R01EB001988-16) and the NSF (19DMS1208164).

{\bf Data availability}. Data are available at \url{https://archive.ics.uci.edu/} (UCI ML repository) and \url{https://www.cbioportal.org/study/summary?id=msk_chord_2024} (MSK-CHORD). The mass spectrometry data are available on request from the corresponding author EC.

\vskip 0.2in

\bibliography{main} 

\appendix
\numberwithin{equation}{section}
\input{supplement}

\end{document}

%% file: supplement.tex
\section{Simulations}
\label{app:examples:sims}
We evaluate attention lasso on four simulated settings designed to test different forms of data heterogeneity. Each setting uses training and test datasets with $n = 300$ observations and $p = 30$ features (except Setting 2, which uses $p=100$). Features are generated as $\bm{X} \sim N(\bm{0}, \bm{I}_p)$ with modifications described below, and responses are generated as $y_i = \bm{x}_i^\top \bm{\beta}_i + \epsilon_i$ where $\epsilon_i \sim N(0, \sigma^2)$ with $\sigma$ chosen to achieve an average signal-to-noise ratio of approximately 2.5. Table~\ref{tab:sim_overview} summarizes the key characteristics of each setting.

\begin{table}[h]
\centering
\begin{tabular}{llp{8cm}}
\toprule
Setting & Type of Heterogeneity & Challenge \\
\midrule
1 & Continuous coefficient variation & Smooth gradient in coefficients \\
2 & High-dimensional continuous variation & Many features ($p=100$) \\
3 & Discrete subgroups with confounders & Minority subgroup (20\%) with spurious covariate shifts \\
4 & Overlapping soft clusters & Three groups with blended membership \\
\bottomrule
\end{tabular}
\caption{\em Overview of simulation settings.}
\label{tab:sim_overview}
\end{table}

\medskip

\noindent\textbf{Setting 1: Continuous coefficient variation.} 
We generate observation-specific coefficients that vary smoothly along a latent gradient. Each observation $i$ has a latent position $z_i \sim \text{Uniform}(-1, 1)$, and coefficients are defined as a convex combination:
\[
\bm{\beta}_i = w_i \bm{\beta}_0 + (1-w_i) \bm{\beta}_1, \quad \text{where } w_i = \frac{z_i + 1}{2}.
\]
Here, $\bm{\beta}_0 = (3, 3, 3, 3, 0, \ldots, 0)^\top$ and $\bm{\beta}_1 = (-2, -2, -2, -2, 0, \ldots, 0)^\top$. Thus, observations with $z_i = -1$ have coefficients $\bm{\beta}_0$, those with $z_i = 1$ have coefficients $\bm{\beta}_1$, and intermediate values smoothly interpolate between them. To create covariate distribution shifts correlated with the coefficient variation, we add $z_i$ to features 1--4: $\bm{x}_{i,j} \leftarrow \bm{x}_{i,j} + z_i$ for $j \in \{1,2,3,4\}$. This setting tests whether attention lasso can adapt to a continuous spectrum of coefficient variation rather than discrete clusters.

\medskip

\noindent\textbf{Setting 2: High-dimensional continuous variation.} 
This extends Setting 1 to higher dimensions ($p=100$) with different coefficient patterns. We again use $z_i \sim \text{Uniform}(0, 1)$ to define smoothly varying coefficients, but now with $\bm{\beta}_0 = (3, 2, 1, 0, 0, 0, 0, \ldots, 0)^\top$ and $\bm{\beta}_1 = (-1, 0, 1, 2, 3, 2, 0, \ldots, 0)^\top$. Features 1--6 are shifted by $z_i$ to create a smooth gradient in the covariate distribution. This setting evaluates performance when the signal must be identified among many noise features.

\medskip

\noindent\textbf{Setting 3: Discrete subgroups with spurious heterogeneity.} 
We partition observations into a majority group (80\%) and a minority group (20\%). The majority has coefficients $\bm{\beta}_1$ with four non-zero entries in positions 1--4; the minority has coefficients $\bm{\beta}_2$ with four non-zero entries in positions 5--8. All non-zero coefficient values are drawn independently from $N(0, 1)$. To ensure the minority group is distinguishable in covariate space, we add a constant shift of 2 to features 1--8 for minority observations. Additionally, we introduce spurious covariate heterogeneity to make subgroup identification more challenging: for a random 50\% of all observations (cutting across both groups), we add random shifts drawn from $N(0, 1)$ to 10 randomly selected noise features. This creates covariate differences that are uncorrelated with the response and do not align with the true subgroup structure, testing whether attention lasso can focus on relevant heterogeneity.

\medskip

\noindent\textbf{Setting 4: Overlapping soft clusters.} 
We generate data with three overlapping subgroups having distinct coefficient vectors: $\bm{\beta}_1 = (3, 3, 2, 1, 0, \ldots, 0)^\top$, $\bm{\beta}_2 = (-2, 1, 3, 2, 0, \ldots, 0)^\top$, and $\bm{\beta}_3 = (1, -2, -1, 3, 0, \ldots, 0)^\top$. Rather than hard cluster assignments, each observation has soft membership in all three groups. We draw a latent position $z_i \sim \text{Uniform}(0, 1)$ and compute membership weights by evaluating three Gaussian density functions centered at 0.2, 0.5, and 0.8 (each with standard deviation 0.15) at $z_i$, then normalizing these densities to sum to 1. Each observation's coefficients are then a weighted blend: $\bm{\beta}_i = w_{i,1}\bm{\beta}_1 + w_{i,2}\bm{\beta}_2 + w_{i,3}\bm{\beta}_3$. We shift features 1--4 by $z_i$ to create covariate patterns correlated with the blended coefficient structure. This setting evaluates whether attention lasso can handle graded, overlapping group memberships rather than discrete clusters.

For each setting, we generate 100 independent train sets with test sets of equal size. Our results are summarized in Table~\ref{tab:sim_results} and Figure~\ref{fig:simplot}. When data are heterogeneous, attention lasso matches or outperforms the lasso as expected. We also find that attention lasso has performance close to that of random forest and XGBoost, but it is much more interpretable. Further, because attention lasso blends a baseline model with the individual weighted model, it does not typically perform \textit{worse} than the lasso, whereas random forest and XGBoost have no such protection.

\begin{table}[htbp]
\centering
\begin{tabular}{l|rrrrrr}
\toprule
Setting & AttnLasso & LLF & LightGBM & XGBoost & RF & KNN \\
\hline
1 & 54.0 (0.7) & 53.9 (0.6) & 51.1 (0.6) & 50.0 (0.8) & \textbf{60.8 (0.5) }& 59.5 (0.6) \\
2 & \textbf{4.1 (0.5) }& -0.5 (1.0) & -16.7 (1.4) & -38.6 (2.2) & -39.7 (1.9) & -46.3 (2.2) \\
3 &\textbf{ 6.6 (0.8)} & -0.3 (1.3) & 3.4 (1.3) & -6.8 (1.7) & 2.0 (1.1) & -12.3 (1.6) \\
4 & 5.8 (0.9) & -2.1 (0.9) & 3.8 (0.7) & 0.4 (1.5) &\textbf{ 8.1 (0.5) }& 5.5 (0.9) \\
\hline
\end{tabular}
\caption{\em Mean (SE) of relative improvement (\%) over the lasso across four data-generating settings (100 simulations each). Simulation details are described in Appendix~\ref{app:examples:sims}.}
\label{tab:sim_results}
\end{table}

\begin{figure}[H]
    \centering
    \includegraphics[width=\linewidth]{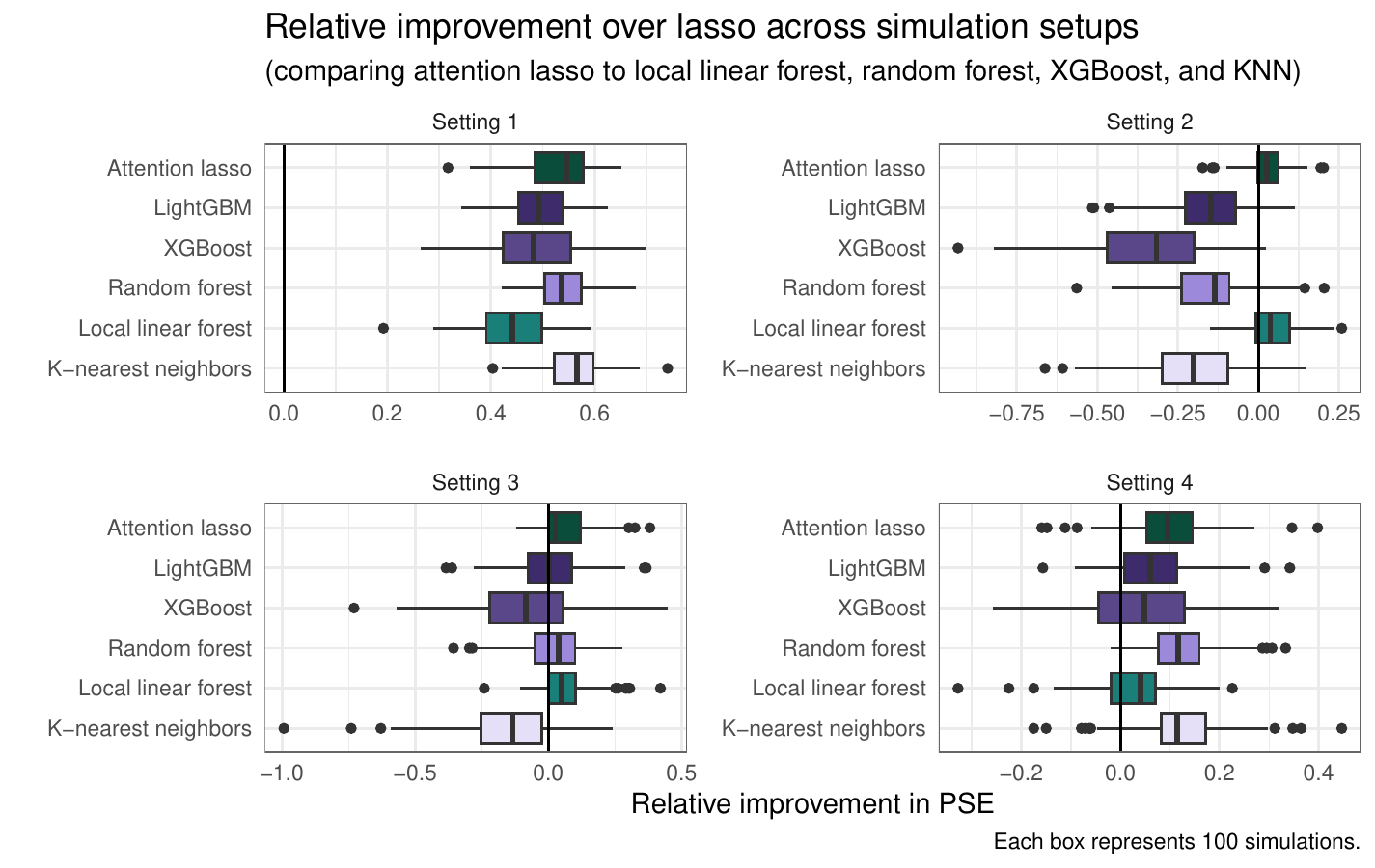}
    \caption{\em Simulation results described in Appendix~\ref{app:examples:sims} and summarized in Table~\ref{tab:sim_results}. Across a variety of simulations, attention lasso typically (1) matches or improves on the lasso and (2) is competitive with more complex models when they perform well. The vertical line at $x = 0\%$ indicates no performance difference from the lasso, and values to the right indicate better performance.}
    \label{fig:simplot}
\end{figure}

\clearpage

\begingroup
\renewcommand{\thesection}{B.\arabic{section}}
\renewcommand{\theHsection}{B.\arabic{section}}
\setcounter{section}{0}
\makeatletter
\phantomsection
\def\@currentlabel{B}
\label{app:theory}
\makeatother
\section*{B\quad Theory for the hard-attention variant}
\input{attention_lasso_body.tex}
\endgroup
\setcounter{section}{2}

\input{kernel}

%% file: attention_lasso_body.tex
\section{Overview and foundations}
\label{att:sec:intro}

This appendix analyzes the hard-attention variant of attention lasso described in
the Theory section of the main text: a query-specific lasso fitted on a
neighborhood selected by an empirical CART tree. The same observations determine the splits and the
local regression. 
all distance ties are retained, features are locally range-normalized, and the
intercept is unpenalized. 
We prove local prediction and coefficient bounds,
exact signs under a local-information condition, and clipped integrated-risk
consistency for a sparse continuous piecewise affine model.

The analysis builds on published sufficient-impurity-decrease (SID) theory.
\citet{ChiEtAl2022} introduced SID and obtained sample-CART approximation and
prediction bounds in high dimension; their Corollary~1, equation~(19), separates
the population approximation error of the fitted partition from estimation
error. \citet{MazumderWang2023} sharpened the CART prediction analysis in
Theorem~2.3, using uniform empirical--population split comparisons in
Lemma~A.4, and related additive-model SID to a locally reverse-Poincare
inequality in Definition~3.1 and Propositions~3.1--3.2. We follow this strategy.

There are two routes to the geometric input. Under the design and bounded-noise
hypotheses of \citet{MazumderWang2023}, their published split-comparison result
can be inserted directly into the depth-to-neighborhood argument below.
Proposition~\ref{att:prop:published-route} states that specialization with its
extra hypotheses. To retain arbitrary independent inactive-coordinate laws and
conditionally sub-Gaussian noise, we also prove a bounded-signal version of the
split comparison from elementary rectangle counts and conditional noise
bounds. This extension and the transfer from a shallow frontier to a
count-ranked neighborhood are proved in this appendix. 

The usual guarantees for the lasso \citep{ZhaoYu2006,Wainwright2009} assume that the observations in the fit were chosen without reference to their responses. Here the tree chooses the node using the responses. We therefore prove that the properties those guarantees need, namely well-behaved features and small noise, hold at once for every possible node, and so for the node the tree selects. With these properties in hand, the lasso argument itself is standard.

\paragraph{Organization.}
Section~\ref{att:sec:model} states the model, the hard-attention procedure, and the geometry of the regression function. Section~\ref{att:sec:results} gives the guarantees: the purity of the neighborhoods, the prediction error of the local lasso and its recovery of support and signs, and the integrated risk. Section~\ref{att:sec:comparison} compares this risk with that of a global affine fit, and Section~\ref{att:sec:proofs} contains the proofs. The analysis covers the hard-attention estimator described in the Theory section of the main text, not the forest-proximity softmax weights or the cross-validated blending of the main method; Section~\ref{att:sec:limits} states the limits in detail.

\section{Model, neighborhood, and local lasso}
\label{att:sec:model}

\subsection{Piecewise affine model and noise}
Let \((X_i,Y_i)_{i=1}^n\) be independent copies of \((X,Y)\), with
\(X\in\mathbb R^p\), \(\E[Y^2]<\infty\), and
\[
Y=f(X)+\xi,\qquad \E[\xi\mid X]=0,\qquad
\E[\xi^2\mid X]=\sigma^2.
\]
The coordinates of \(X\) are independent. An active set
\(S\subseteq\{1,\ldots,p\}\) has cardinality \(s\ge1\); for \(j\in S\),
\(X_j\) is uniform on a nondegenerate compact interval \([a_j,b_j]\).
Inactive coordinates have arbitrary laws. They need not have finite
moments. Put \(L_j=b_j-a_j\) and
\(Q_S=\prod_{j\in S}[a_j,b_j]\). Assume
\(f(X)=F(X_S)\) almost surely, where \(F\) is continuous on \(Q_S\).
There is a finite partition, up to null sets, into pieces \(G_g\), each a
finite union of axis-aligned rectangles, such that
\[
F(x)=a_g+\sum_{j\in S}\beta_{gj}x_j
\quad\text{for Lebesgue-a.e. }x\in G_g,
\qquad
\lambda_j\le|\beta_{gj}|\le R\lambda_j
\quad(j\in S)
\label{att:eq:model}
\]
for numbers \(\lambda_j>0\) and \(R\ge1\). We set \(\beta_{gj}=0\) for
\(j\notin S\). In particular, all pieces have the same support \(S\),
although their slopes and signs may differ.

Refining by all component-rectangle endpoints gives a finite product grid.
Let \(J\) be the largest number of its intervals in one active coordinate,
and let \(B_\partial\) be the total number of internal cutpoints over the
active coordinates. Thus \(B_\partial\le s(J-1)\). Let \(\chi(x)\) be the
grid-cell label of the active query, and \(g(x)\) its original piece label.
Labels can be extended arbitrarily on the finitely many grid hyperplanes;
almost surely neither the training data nor an independent query is on
them. Write
\[
V_0=\Var[f(X)],\qquad
\lambda_* =\min_{j\in S}\lambda_jL_j.
\]

The noise satisfies the conditional exponential-square condition: when \(\sigma>0\),
\begin{equation}\label{att:eq:noise-square}
\E\left[\exp\{\theta_\xi\xi^2/\sigma^2\}\mid X\right]\le K_\xi
\quad\text{almost surely}
\end{equation}
for \(\theta_\xi>0\) and \(K_\xi\ge1\). When \(\sigma=0\), \(\xi=0\)
almost surely. For tuning, choose any positive \(v\) such that
\begin{equation}\label{att:eq:noise-mgf}
\E[e^{t\xi}\mid X]\le e^{v^2t^2/2},\qquad t\in\mathbb R.
\end{equation}
This adds no distributional hypothesis: when \(\sigma>0\), one may take
\(v^2=8K_\xi\sigma^2/\theta_\xi\), as proved in
Lemma~\ref{att:lem:noise-mgf}. In the noiseless case any positive \(v\) is
valid. Selecting a valid upper noise scale is a tuning requirement; no
claim about automatic or cross-validated tuning is made.

Grow the following empirical CART tree on these same data, with
stopping threshold \(1\le k<n\). At a node it minimizes empirical within-child
squared error over coordinate thresholds with two nonempty empirical
children; ties are resolved by the smallest coordinate index and then the
smallest observed threshold. It stops when the node count is at most \(k\)
or no split is available. In particular, \(k\) is not a minimum child size.
Let \(\mathcal D_n\) denote the training sample, \(\mu=P_X\), and
\(N_n(A)=\sum_i\mathbf1_{\{X_i\in A\}}\).

For \(1\le\ell\le n\), let \(\mathcal I_{n,\ell}(x)\) be the indices of
all observations whose tree distance is no larger than the \(\ell\)-th
smallest distance, where the distance is the count in the deepest common
node. Define
\begin{equation}\label{att:eq:bad-probability}
q_{n,\ell}=P_X\{x:\exists i\in\mathcal I_{n,\ell}(x),
                            \ \chi(X_i)\ne\chi(x)\},
\qquad \pi_{n,\ell}=\E[q_{n,\ell}].
\end{equation}
The first probability is evaluated with the training sample fixed.
We prove a bound for these probabilities in
Theorem~\ref{att:thm:purity}; purity is not an assumption.
For definiteness, all nodes are half-open rectangles, and the finite
candidate list at a node uses its observed coordinate values as thresholds.
Every available threshold has the sample partition of one such candidate.
Every selected split has two nonempty empirical children, so counts
strictly decrease down a branch and the construction terminates. All
candidate comparisons, node endpoints, and distance ranks are measurable.

\subsection{The precise attention procedure}
For a query \(x\), let \(A=A_\ell(x)\) be the deepest node on its
root-to-leaf path with \(N_n(A)\ge\ell\). Set \(N=N_n(A)\) and
\(I_A=\{i:X_i\in A\}\). Lemma~\ref{att:lem:ties} below shows that
\(I_A=\mathcal I_{n,\ell}(x)\). The hard-attention weights are
\begin{equation}\label{att:eq:weights}
W_i(x)=N^{-1}\mathbf1_{\{i\in I_A\}}.
\end{equation}
They are uniform on this local rectangle, not on the whole affine piece.

For each coordinate set
\begin{align}
c_j(x)&=\min\bigl(\{X_{ij}:i\in I_A\}\cup\{x_j\}\bigr),\nonumber\\
r_j(x)&=\max\bigl(\{X_{ij}:i\in I_A\}\cup\{x_j\}\bigr)-c_j(x),
\label{att:eq:ranges}\\
\overline X_{Aj}&=N^{-1}\sum_{i\in I_A}X_{ij},\qquad
\overline Y_A=N^{-1}\sum_{i\in I_A}Y_i.\nonumber
\end{align}
Let \(D(x)=\{j:r_j(x)>0\}\). For \(j\in D(x)\), define
\begin{equation}\label{att:eq:features}
Z_{ij}=\frac{X_{ij}-\overline X_{Aj}}{r_j(x)},\qquad
z_j(x)=\frac{x_j-\overline X_{Aj}}{r_j(x)}.
\end{equation}
Both selected training coordinates and query coordinates are bounded in
absolute value by one after this normalization. Columns with \(r_j=0\)
are omitted and assigned zero slope.

Let
\begin{equation}\label{att:eq:tuning}
a_n=p\log(2en)+3\log n,\qquad
\tau_A=K v\sqrt{a_n/N},
\end{equation}
where \(K\) is a fixed, sufficiently large numerical constant. Fit
\begin{equation}\label{att:eq:lasso}
(\widehat\eta,\widehat\theta)=
\arg\min_{\eta\in\mathbb R,\ \theta\in\mathbb R^{D(x)}}
\left\{\frac1{2N}\sum_{i\in I_A}
  (Y_i-\eta-Z_i^\mathsf T\theta)^2+\tau_A\|\theta\|_1\right\}.
\end{equation}
The intercept is unpenalized, so \(\widehat\eta=\overline Y_A\). If the
minimizer is not unique, select the one with smallest Euclidean norm in
\(\theta\). Define
\begin{equation}\label{att:eq:output}
\widehat\beta_j(x)=\widehat\theta_j/r_j(x)\quad(j\in D(x)),\qquad
\widehat f(x)=\overline Y_A+\widehat\theta^\mathsf Tz(x),
\end{equation}
with the omitted slopes zero. The corresponding affine intercept is
\(\widehat a(x)=\overline Y_A-\widehat\beta(x)^\mathsf T\overline X_A\).
Existence and measurability of the selected solution are justified in
Section~\ref{att:subsec:measurable}.

\subsection{Population geometry and SID}
\label{att:subsec:population}
For a rectangle \(A\) of positive population mass write
\[
 m(A)=\E[f(X)\mid X\in A],\qquad
 V(A)=\Var[f(X)\mid X\in A].
\]
Set \(\mu(A)V(A)=0\) when \(\mu(A)=0\). For a coordinate-threshold split
\(t=(j,b)\), with children \(A_t^-\) and \(A_t^+\), let
\begin{equation}\label{att:eq:gain-definition}
 \Delta(A,t)=\mu(A)V(A)
       -\mu(A_t^-)V(A_t^-)-\mu(A_t^+)V(A_t^+).
\end{equation}
This is the globally mass-weighted population gain, in the convention of
\citet{MazumderWang2023}. A number \(\lambda\in(0,1]\) is an SID constant if
\begin{equation}\label{att:eq:sid}
 \sup_t\Delta(A,t)\ge\lambda\mu(A)V(A)
       \quad\text{for every rectangle }A.
\end{equation}
We use \(\lambda\) only for SID, \(\lambda_j\) for coordinate slope scales,
and \(\tau_A\) for the lasso penalty.

\begin{proposition}[Geometry of the specified piecewise affine model]
\label{att:prop:geometry}
Under the model in Section~\ref{att:sec:model}, there are continuous
piecewise affine univariate functions \(g_j\) such that
\(F(x)=c_0+\sum_{j\in S}g_j(x_j)\). Each \(g_j\) has at most \(J\) grid
intervals and slopes with magnitudes in \([\lambda_j,R\lambda_j]\).
Put \(\Gamma=J\) in general. If the slope of each coordinate has a fixed
sign across all pieces, put \(\Gamma=1\). In the two cases one may use,
respectively,
\begin{equation}\label{att:eq:sid-constants}
 \lambda=\frac{1}{3sJ^2R^2},
 \qquad\text{or}\qquad
 \lambda=\frac{3}{4sR^2}.
\end{equation}
For the conditional active side lengths \(w_j(A)\), every positive-mass
rectangle satisfies
\begin{equation}\label{att:eq:geometry-bound}
 V(A)\ge\frac{1}{12\Gamma^2}
            \sum_{j\in S}\lambda_j^2w_j(A)^2
       \ge\frac{\lambda_*^2}{12\Gamma^2}\omega(A)^2,
 \qquad
 \omega(A)=\max_{j\in S}\frac{w_j(A)}{L_j}.
\end{equation}
Also, for
\(B_f=\max_{x\in Q_S}|F(x)-\E[f(X)]|\),
\begin{equation}\label{att:eq:root-range}
 B_f^2\le12s\Gamma^2R^2V_0.
\end{equation}
No minimum length or mass of an individual piece is required.
\end{proposition}

The verification is given in Section~\ref{att:subsec:geometry-proof}.
Its integration-by-parts step is the reverse-Poincare-to-SID mechanism used
by \citet{MazumderWang2023}. For example, the present univariate components
satisfy, on every interval \(I\),
\begin{equation}\label{att:eq:lrp-specialization}
 \left(\int_I|g_j'(u)|\,du\right)^2
 \le12R^2\Gamma^2\Var[g_j(U_I)]
 =\frac{12R^2\Gamma^2}{|I|}\inf_c\int_I(g_j(u)-c)^2\,du,
\end{equation}
with derivatives interpreted almost everywhere. We prove the needed
piecewise version, rather than imposing differentiability at breakpoints
or importing a minimum-segment-length hypothesis.

\section{Main guarantees}
\label{att:sec:results}
For matrices, \(\|H\|_\infty\) denotes maximum absolute row sum and
\(\|H\|_{\max}\) maximum absolute entry. For vectors, the same infinity
symbol denotes maximum absolute coordinate. The active support lengths
\(L_j\) are included in coefficient-error statements to make them invariant
to changes of the covariate units. Constants below do not depend on the laws
or magnitudes of inactive coordinates.

\subsection{Localization of the original tree neighborhoods}
\label{att:subsec:localization-statements}
Let \(v_0^2=8K_\xi\sigma^2/\theta_\xi\) when \(\sigma>0\), and set
\(v_0=0\) when \(\sigma=0\). This is a valid conditional sub-Gaussian
proxy, by Lemma~\ref{att:lem:noise-mgf}; unlike the positive tuning bound
\(v\), it can be zero. Define
\begin{equation}\label{att:eq:localization-scales}
 H^2=B_f^2+v_0^2,\qquad
 z_n=\frac{k+\ell+a_n}{n},\qquad
 \vartheta=1-\lambda/2,\qquad
 \kappa=-\tfrac12\log_2\vartheta>0.
\end{equation}
By Proposition~\ref{att:prop:geometry},
\(H^2\le C(s,\Gamma,R,\theta_\xi,K_\xi)(V_0+\sigma^2)\).
These quantities are used for analysis, not supplied to the local fit.

\begin{theorem}[Depth-to-neighborhood purity under the product model]
\label{att:thm:purity}
Under the model of Section~\ref{att:sec:model}, there is a numerical
constant \(C\) and a training event \(\mathcal G_n\) with
\(\mathbb P(\mathcal G_n)\ge1-Cn^{-2}\) such that, simultaneously for
all integers \(D\ge0\) and all \(r>0\),
\begin{equation}\label{att:eq:purity-depth-radius}
 q_{n,\ell}\le
 C\frac{2^D(k+\ell+a_n)}{n}+2B_\partial r
 +\frac{C\Gamma^2}{\lambda_*^2r^2}
       \left\{\vartheta^D V_0+H^2\frac{2^Da_n}{n}\right\}.
\end{equation}
If \(B_\partial>0\) and \(z_n\le1/4\), take
\(D=\lfloor\tfrac12\log_2(1/z_n)\rfloor\). Then, on the same event,
\begin{equation}\label{att:eq:proved-conditional}
 q_{n,\ell}\le1\wedge\left[
 Cz_n^{1/2}+C\left\{
 \frac{B_\partial^2\Gamma^2}{\lambda_*^2}
     \left(V_0z_n^\kappa+H^2z_n^{1/2}\right)
                    \right\}^{1/3}\right]=:\Psi_n,
\end{equation}
and
\begin{equation}\label{att:eq:proved-purity}
 \pi_{n,\ell}\le Cn^{-2}+\Psi_n.
\end{equation}
If \(B_\partial=0\), all grid labels agree almost surely and
\(q_{n,\ell}=\pi_{n,\ell}=0\), without a smallness condition. The original
count-stopped tree is not pruned or otherwise changed. Its stopping
threshold may be \(k=1\).
\end{theorem}

The proof follows the published SID split-comparison strategy but includes
the distribution-free bounded-signal extension needed for this model. 
The term \(2^D(k+\ell+a_n)/n\) accounts for early terminal paths
and depth-\(D\) nodes containing fewer than \(\ell\) observations.

\begin{corollary}[Asymptotic purity]
\label{att:cor:purity}
Suppose \(s,J,R\) are uniformly bounded, the noise-square constants are
uniformly controlled, and
\begin{equation}\label{att:eq:purity-asymptotics}
 \sup_n\frac{V_{0,n}+\sigma_n^2}{\lambda_{*,n}^2}<\infty,
 \qquad p_n\log n=o(n),\qquad k_n=o(n),\qquad\ell_n=o(n).
\end{equation}
Then \(q_{n,\ell_n}\to0\) in probability and
\(\pi_{n,\ell_n}\to0\). The conclusion includes all ties and does not
require \(\ell_n\to\infty\). For fixed-sign slopes, \(J_n\) may grow if
\[
 \frac{B_{\partial,n}^2}{\lambda_{*,n}^2}
       \{V_{0,n}z_n^\kappa+H_n^2z_n^{1/2}\}\longrightarrow0,
\]
with a common positive \(\kappa\) chosen from fixed bounds on \(s,R\).
\end{corollary}

\subsection{A direct specialization of the published CART bound}
\label{att:subsec:published-route}
The preceding theorem retains the full product model. The next proposition
records a shorter route when the assumptions of the published result itself
hold. It is not used to justify the broader noise or covariate laws in
Theorem~\ref{att:thm:purity}.

\begin{proposition}[Published-theorem route on the common subclass]
\label{att:prop:published-route}
In addition to the model of Section~\ref{att:sec:model}, suppose that,
after increasing affine coordinate rescalings, all covariates are supported
on \([0,1]^p\) and have a joint density between fixed positive bounds
\(\underline\theta\) and \(\overline\theta\). Suppose the noises are
independent of the covariates, independent and identically distributed,
centered, and bounded by \(b_\xi\). Put \(U=B_f+b_\xi\). For an integer
\(D\ge1\), assume
\begin{equation}\label{att:eq:mw-size}
 \max\left\{\frac{2\overline\theta e^2D}{n},\,
    \frac{\log\{72p^D(n+1)^{2D}n^2\}}{n}\right\}<\frac34.
\end{equation}
Then on a training event of probability at least \(1-Cn^{-2}\),
\begin{equation}\label{att:eq:mw-purity}
 q_{n,\ell}\le
 C\frac{2^D(k+\ell+a_n)}n+
 C\left\{\frac{B_\partial^2\Gamma^2}{\lambda_*^2}
 \left[
 V_0(1-\lambda/4)^D+
 CU^2\frac{2^D\{D\log(2np)+\log n\}}n
 \right]\right\}^{1/3},
\end{equation}
when \(B_\partial>0\); probabilities are also bounded by one. The
constants in this proposition may depend on the density bounds.
\end{proposition}

Here the finite-depth contraction is obtained directly from
\citet[Lemma~A.4 and the proof of Theorem~2.3]{MazumderWang2023}, with their
comparison parameter equal to one and failure probability \(n^{-2}\).
Only the depth-to-neighborhood transfer and the local lasso calculations
remain to be supplied. Theorem~2.3 itself states a prediction bound with
leading approximation term \(2V_0(1-\lambda/4)^D\); population within-cell
variation is no larger than the prediction error of any constant-on-cell
fit. Using their uniform split lemma also handles paths stopped by our
count threshold without any assertion that the two complete trees are
identical. \citet[Corollary~1, equation~(19)]{ChiEtAl2022} provides another
approximation input under that paper's own hypotheses. We do not import
its randomized tie convention into our deterministically tie-broken
procedure; Proposition~\ref{att:prop:published-route} uses the uniform
maximizer comparison of \citet{MazumderWang2023}.

\subsection{Local lasso guarantees}

\begin{theorem}[Local regression, coefficient error, and signs]
\label{att:thm:local}
Assume the model of Section~\ref{att:sec:model}. There are numerical constants
\(C_0,C_1,K_0\) such that the following holds when
\(\ell\ge C_0s^2a_n\) and \(K\ge K_0\) in
\eqref{att:eq:tuning}. There is a training-sample event \(\mathcal E_n\)
with \(P(\mathcal E_n)\ge1-C_1n^{-2}\) on which, simultaneously for every
query \(x\) whose tie-inclusive neighborhood has the same grid label as
\(x\), the lasso solution is unique and
\begin{align}
\supp\widehat\beta(x)&\subseteq S,
\label{att:eq:no-false-positives}\\
\max_{j\in S}r_j(x)|\widehat\beta_j(x)-\beta_{g(x),j}|
 &\le C\tau_A,
\label{att:eq:scaled-error}\\
|\widehat f(x)-f(x)|&\le C(s+1)v\sqrt{a_n/N}
 \le C(s+1)v\sqrt{a_n/\ell},
\label{att:eq:query-error}\\
\max_{j\in S}L_j|\widehat\beta_j(x)-\beta_{g(x),j}|
 &\le Cv\frac n\ell\sqrt{a_n/\ell}.
\label{att:eq:coefficient-error}
\end{align}
The constant \(C\) may depend on the chosen fixed penalty multiplier \(K\).
Here and throughout, queries are restricted to a fixed set of full
covariate probability on which the model representation and labels are
valid. If additionally
\begin{equation}\label{att:eq:finite-beta-min}
\lambda_*\frac\ell n>Cv\sqrt{a_n/\ell},
\end{equation}
then, for all such queries,
\begin{equation}\label{att:eq:signs}
\supp\widehat\beta(x)=S,
\qquad
\sgn\widehat\beta_j(x)=\sgn\beta_{g(x),j}
\quad(j\in S).
\end{equation}
For an independent query, the joint probability of all conclusions not
requiring \eqref{att:eq:finite-beta-min} is at least
\(1-\pi_{n,\ell}-C_1n^{-2}\); the sign conclusion has the same probability
when \eqref{att:eq:finite-beta-min} holds. Conditional on
\(\mathcal E_n\), the conclusions fail on a query set of probability at
most \(q_{n,\ell}\).
\end{theorem}

The coefficient error of the vector
\((L_j(\widehat\beta_j-\beta_{g,j}))_{j\in S}\) in \(\ell^2\) or
\(\ell^1\) norm is at most \(\sqrt{s}\) or \(s\) times the right side
of \eqref{att:eq:coefficient-error}. The theorem does not require an assumed
empirical Gram or irrepresentable event. The price of uniformity over
response-selected rectangles is the conservative complexity \(a_n\) of
order \(p\log n\), rather than just \(\log p\).

\begin{corollary}[Asymptotic local recovery]
\label{att:cor:asymptotic}
Consider a sequence of models with uniformly bounded \(s,J,R\), controlled
noise-tail constants, and
\[
\sup_n\frac{V_{0,n}+\sigma_n^2}{\lambda_{*,n}^2}<\infty,\qquad
p_n\log n=o(n),\qquad k_n=o(n),\qquad\ell_n=o(n).
\]
If \(s_n^2a_n/\ell_n\to0\), the conclusions
\eqref{att:eq:no-false-positives}--\eqref{att:eq:coefficient-error} hold with
probability tending to one for a fresh query, with the displayed rates.
If moreover
\begin{equation}\label{att:eq:information-condition}
\frac{v_n^2n^2a_n}{\lambda_{*,n}^2\ell_n^3}\longrightarrow0,
\end{equation}
then exact local sign and support recovery hold with probability tending
to one. Absolute coefficient consistency in the support-length-normalized
units additionally follows whenever
\(v_n^2n^2a_n/\ell_n^3\to0\).
For a fixed active model and bounded \(v\), one compatible choice is
\begin{equation}\label{att:eq:ell-choice}
\ell_n=\left\lceil n d_n^{1/4}\right\rceil,
\qquad d_n=p_n\log n/n.
\end{equation}
It yields normalized coefficient error \(O_P(vd_n^{1/8})\) and same-query
squared error \(O_P((s+1)^2v^2d_n^{3/4})\). These are sufficient rates, not
optimality claims; the integrated risk also contains the purity error
below.
\end{corollary}

\begin{theorem}[Integrated prediction risk]
\label{att:thm:risk}
Under the assumptions of Theorem~\ref{att:thm:local}, suppose
\(|f(X)|\le M<\infty\) almost surely, and let
\(\widehat f^M=\max\{-M,\min\{M,\widehat f\}\}\). On its event
\(\mathcal E_n\),
\begin{equation}\label{att:eq:risk-conditional}
\|\widehat f^M-f\|_{L^2(P_X)}^2
\le C(s+1)^2v^2\frac{a_n}{\ell}+4M^2q_{n,\ell}.
\end{equation}
Consequently,
\begin{equation}\label{att:eq:risk-expectation}
\E\bigl[\|\widehat f^M-f\|_{L^2(P_X)}^2\bigr]
\le C(s+1)^2v^2\frac{a_n}{\ell}
       +4M^2\{\pi_{n,\ell}+C_1n^{-2}\}.
\end{equation}
For an independent test pair, the conditional squared prediction risk is
\(\sigma^2+\|\widehat f^M-f\|_2^2\). Thus it tends to \(\sigma^2\) in
probability, and also after averaging over training samples, whenever the
right side of \eqref{att:eq:risk-expectation} tends to zero.
\end{theorem}

A finite bound \(M\) exists for every fixed active model by compactness and
continuity; clipping is only used to control exceptional queries and
training samples in integrated risk. It does not affect any coefficient or
support assertion. For a fixed model it is unnecessary to know \(M\):
when \(z_n\to0\), one can use
\(M_n=\sqrt{\log(e/z_n)}\). Eventually this covers the fixed range of \(f\),
and the proved purity bound implies \(M_n^2\pi_{n,\ell}\to0\) under
bounded model parameters and relative signal. This data-independent choice
is justified in the proof of Theorem~\ref{att:thm:risk}.

\begin{remark}[An explicit integrated bound]
\label{att:rem:explicit-risk}
Combining Theorems~\ref{att:thm:purity} and \ref{att:thm:risk}, when
\(z_n\le1/4\), gives
\[
 \E\bigl[\|\widehat f^M-f\|_2^2\bigr]
 \le C(s+1)^2v^2\frac{a_n}{\ell}+CM^2\{n^{-2}+\Psi_n\}.
\]
For fixed model parameters and bounded relative signal,
\(\Psi_n\le C\{z_n^{1/2}+z_n^{\kappa/3}+z_n^{1/6}\}\).
This bound, rather than a conditional coefficient rate alone, is used for
integrated prediction comparisons. The powers are conservative; no claim
of an optimal rate or an improvement on the published CART prediction
bounds is made.
\end{remark}

\begin{remark}[Why coefficient recovery asks for more]
\label{att:rem:information}
The local slope signal is \(|\beta_{g,j}|r_j\), not merely
\(|\beta_{g,j}|\). Inequality~\eqref{att:eq:range-lower} of Lemma~\ref{att:lem:gram} gives only
\(r_j/L_j\ge cN/n\) uniformly over populated rectangles. Inverting this
scale gives the factor \(n/\ell\) in
\eqref{att:eq:coefficient-error}. For example, a one-dimensional uniform
window of width of order \(\ell/n\) has centered squared-design sum of
order \(\ell^3/n^2\). It can estimate its mean well before it can estimate
a slope precisely. Condition~\eqref{att:eq:information-condition} is a
worst-case sufficient local-information condition, not an extra assumption
that the tree has a prescribed aspect ratio. Prediction needs neither
\eqref{att:eq:finite-beta-min} nor slope consistency.
\end{remark}

\section{Comparison with a global affine fit}
\label{att:sec:comparison}

\begin{proposition}[Approximation gap for every global affine predictor]
\label{att:prop:global-gap}
Define
\begin{equation}\label{att:eq:gap}
\Delta_f=\inf_{a\in\mathbb R,\ b\in\mathbb R^s}
 \E\bigl[\{f(X)-a-b^\mathsf TX_S\}^2\bigr].
\end{equation}
Every training-sample-dependent global affine predictor
\(\widehat f_G(x)=\widehat a_G+\widehat b_G^\mathsf Tx\) satisfies
\begin{equation}\label{att:eq:global-risk-lower}
R_G(\mathcal D_n):=
\E[(Y-\widehat f_G(X))^2\mid\mathcal D_n]
\ge\sigma^2+\Delta_f,
\end{equation}
with extended-real risk allowed. If the active function \(F\) is fixed and
not affine on the full box, then \(\Delta_f>0\). Under the consistency
conditions of Theorem~\ref{att:thm:risk},
\begin{equation}\label{att:eq:global-comparison}
P\{R_{\mathrm{att}}(\mathcal D_n)<R_G(\mathcal D_n)\}\longrightarrow1,
\end{equation}
where \(R_{\mathrm{att}}\) uses the clipped local predictor and both risks
integrate over the same independent test distribution.
For changing models it suffices that the right side of
\eqref{att:eq:risk-expectation} be \(o(\Delta_{f,n})\).
No consistency or support assumption on the global lasso is needed for
this comparison, because any of its fits is a global affine function.
\end{proposition}

In the additive representation \(f(X)=c_0+\sum_{j\in S}g_j(X_j)\)
proved in Proposition~\ref{att:prop:geometry}, the gap is
\begin{equation}\label{att:eq:additive-gap}
\Delta_f=\sum_{j\in S}\left[
 \Var[g_j(X_j)]-
 \frac{\operatorname{Cov}(X_j,g_j(X_j))^2}{\Var[X_j]}
 \right].
\end{equation}
This comparison is not a pointwise-in-query dominance statement. In
particular, a global affine fit may happen to agree with the regression
function at some queries or on some pieces. If \(F\) is affine,
\(\Delta_f=0\) and no strict advantage follows from this argument.

\begin{example}[Two pieces, cancellation, and a strict risk gap]
\label{att:ex:vee}
Let all covariates be independent uniforms on \([0,1]\), and let
\[
Y=|X_1-1/2|+\xi
\]
with independent centered sub-Gaussian noise of variance \(\sigma^2\).
The model has \(s=1,J=2,R=1\), common support \(\{1\}\), and local
slopes \(-1\) and \(+1\). Its best global affine predictor is the constant
\(1/4\), and \(\Delta_f=1/48\). With a neighborhood size satisfying
\eqref{att:eq:information-condition}, the local procedure recovers the
support \(\{1\}\) and the correct slope sign for an independent query
with probability tending to one. Its clipped prediction risk tends to
\(\sigma^2\).

By contrast, the global lasso with an unpenalized intercept, original
\([0,1]\)-scale columns, and penalty
\(C(1+v)\sqrt{\log(pn)/n}\), with sufficiently large \(C\), has all
slopes zero with probability tending to one. Its risk tends to
\(\sigma^2+1/48\). This is an example of a support advantage through
cancellation, not a general claim that global lasso must select the wrong
support. Under fixed-sign slopes, it can select the correct common support
while still having a nonlinear approximation error.
\end{example}

\section{Proofs}
\label{att:sec:proofs}

\subsection{Piecewise geometry and verification of SID}
\label{att:subsec:geometry-proof}
\begin{proof}[Proof of Proposition~\ref{att:prop:geometry}]
\emph{Additive representation.}
Refine \(Q_S\) by all component-rectangle endpoints. On every open grid
cell, membership in each piece is constant. Because the pieces partition
up to null sets, the cell belongs to exactly one positive-volume piece,
and its affine formula equals \(F\) everywhere there and on its closure,
by continuity. If two cells share a face normal to coordinate \(h\), their
affine formulas agree on the face. Their coefficients in each tangential
coordinate \(j\ne h\) must therefore agree.

Fix \(j\) and one grid interval in that coordinate. All cells having this
same interval are connected by moves across faces normal to coordinates
other than \(j\). Hence their \(j\)-coefficients are identical. Let
\(b_j(t)\) be this coefficient, as a step function of the \(j\)-coordinate,
and set
\[
 g_j(t)=\int_{a_j}^t b_j(u)\,du.
\]
On each open cell, \(F-\sum_jg_j\) has zero gradient and is constant.
Continuity across faces and connectedness of the product grid make the
constants equal. Continuity on the outer boundary proves the representation
on all of \(Q_S\). The asserted slope bounds and number of intervals follow.

\emph{Variance and oscillation.}
For a component \(g_j\) and a conditional uniform interval \(I\) of length
\(w>0\), let \(w_1,\ldots,w_q\), \(q\le J\), be its positive intersection
lengths with the grid intervals. Variance decomposition yields
\[
 \Var[g_j(U_I)]
 \ge\sum_{r=1}^q\frac{w_r}{w}\frac{\lambda_j^2w_r^2}{12}
 \ge\frac{\lambda_j^2w^2}{12J^2},
\]
since \(\sum_rw_r^3\ge w^3/q^2\). Continuity and the upper slope bound
make \(g_j\) globally \(R\lambda_j\)-Lipschitz. Thus, by the independent-copy
variance identity,
\[
 \Var[g_j(U_I)]
 =\tfrac12\E\bigl[\{g_j(U_I)-g_j(U'_I)\}^2\bigr]
 \le\frac{R^2\lambda_j^2w^2}{12}.
\]
If its slopes have a fixed sign, summing the lower slope bound along an
interval gives \(|g_j(u)-g_j(v)|\ge\lambda_j|u-v|\), so the first lower
bound improves to \(\lambda_j^2w^2/12\).

Conditioning a product law on a rectangle preserves independence. Consequently
\[
 V(A)=\sum_{j\in S}\Var[g_j(X_j)\mid X_j\in A_j],
 \qquad
 \frac{\sum_j\lambda_j^2w_j(A)^2}{12\Gamma^2}
 \le V(A)\le\frac{R^2\sum_j\lambda_j^2w_j(A)^2}{12}.
\]
This proves \eqref{att:eq:geometry-bound}. The oscillation of \(F\) on the
conditional active rectangle is at most \(R\sum_j\lambda_jw_j(A)\).
Its conditional mean lies in its range, whence
\[
 |F(x)-m(A)|^2\le sR^2\sum_j\lambda_j^2w_j(A)^2
                 \le12s\Gamma^2R^2V(A).
\]
Apply this at the full support to obtain \eqref{att:eq:root-range}.
It also proves \eqref{att:eq:lrp-specialization}.

\emph{A threshold for a component whose slopes may change sign.}
Write \(I=[a,b]\), \(w=b-a\), \(m=\E[g_j(U_I)]\),
\(v_j=\Var[g_j(U_I)]>0\), and
\[
 H_j(t)=\int_a^t\{g_j(u)-m\}\,du.
\]
The function \(g_j\) is absolutely continuous, \(H_j(a)=H_j(b)=0\), and
integration by parts gives
\[
 wv_j=-\int_a^b g_j'(t)H_j(t)\,dt
      \le R\lambda_jw\|H_j\|_\infty.
\]
The continuous \(|H_j|\) attains its positive maximum at an interior
point \(t_*\). At a split \(t\), the two conditional means are
\(m+H_j(t)/(t-a)\) and \(m-H_j(t)/(b-t)\); the conditional between-child
variance is therefore
\[
 \frac{H_j(t)^2}{(t-a)(b-t)}.
\]
At \(t_*\), it is at least
\[
 \frac{4v_j^2}{R^2\lambda_j^2w^2}\ge\frac{v_j}{3J^2R^2}.
\]
Choose a component with \(v_j\ge V(A)/s\). The other component means
are unaffected by this coordinate split, proving the first SID constant
in \eqref{att:eq:sid-constants} after multiplication by \(\mu(A)\).

\emph{Fixed-sign improvement.}
Choose \(j\) maximizing \(\lambda_j^2w_j(A)^2\). The variance upper bound
implies
\[
 \lambda_j^2w_j(A)^2\ge\frac{12V(A)}{sR^2}.
\]
Splitting its interval at the midpoint gives equal-mass children, whose
means differ in absolute value by at least \(\lambda_jw_j(A)/2\).
The decrease is at least \(\lambda_j^2w_j(A)^2/16\), hence at least
\(3V(A)/(4sR^2)\). This proves the second SID constant.
\end{proof}

\subsection{An elementary relative rectangle inequality}
\label{att:subsec:relative-counts}
We give a direct proof of the count bound needed below. It is a finite-bracket
and scalar-concentration argument in the classical uniform-convergence
framework of \citet{VapnikChervonenkis1971}.

\begin{lemma}[Finite indicator brackets, including atoms]
\label{att:lem:indicator-brackets}
For a Borel probability measure \(Q\) on \(\mathbb R^q\) and
\(0<\varepsilon\le1\), there are at most
\((Cq/\varepsilon)^{4q}\) pairs \((B^-,B^+)\) of measurable rectangles,
allowing all endpoint conventions and the empty set, such that every
rectangle \(B\) is contained in a pair satisfying
\[
 B^-\subseteq B\subseteq B^+,
 \qquad Q(B^+\setminus B^-)\le\varepsilon.
\]
\end{lemma}
\begin{proof}
Set \(\eta=\varepsilon/(4q)\). For each coordinate marginal, isolate
all atoms of mass greater than \(\eta\) as singleton cells. There are
at most \(1/\eta\) of them. In each intervening open interval, including
the two unbounded end intervals, partition consecutively from the left.
If the remaining mass exceeds \(\eta\), end the next cell at the first
point where the accumulated restricted mass reaches \(\eta\). The
endpoint exists inside the interval, and its atom has mass at most
\(\eta\). This cell has mass between \(\eta\) and \(2\eta\).
Continue; the final remainder has mass at most \(\eta\). Zero-mass
remainders are kept when necessary to cover the whole line. The number
of nonfinal cells is at most \(1/\eta\), and there is at most one final
cell per intervening interval. Including the heavy singleton cells,
the resulting ordered partition has at most \(C/\eta\) cells.

For any interval, take the union of the cells wholly contained in it and
the union of the cells meeting it as its lower and upper brackets. These
unions are intervals, possibly empty. Singleton atoms are included exactly
when the original interval includes their locations. Only at most two
nonsingleton boundary cells contribute to the bracket difference, giving
marginal mass at most \(4\eta\). The total number of pairs is at most
\((C/\eta)^4\), a deliberately coarse bound obtained by counting the
endpoints of both brackets. Take products of these one-dimensional
brackets. Their difference lies in the union of the \(q\) coordinate
boundary strips, so its \(Q\)-mass is at most \(4q\eta=\varepsilon\).
Independence of the coordinates is not used.
\end{proof}

\begin{lemma}[Uniform relative counts]
\label{att:lem:relative-counts}
Let \(n\ge2\), and let \(U_1,\ldots,U_n\) be independent with Borel law \(Q\) on
\(\mathbb R^q\), and let \(Q_n\) be the empirical measure. For
\(0<\delta<1\), with probability at least \(1-\delta\),
\begin{equation}\label{att:eq:relative-counts}
\begin{gathered}
 |Q_n(B)-Q(B)|\le C\left\{
  \sqrt{Q(B)\frac{b_{n,q,\delta}}n}+
                 \frac{b_{n,q,\delta}}n\right\}
 \quad\text{for every rectangle }B,\\
 b_{n,q,\delta}=q\log(2en)+\log(2/\delta).
\end{gathered}
\end{equation}
The event can be taken measurable and valid for all endpoint conventions.
\end{lemma}
\begin{proof}
If \(q>n\), the second term, with a numerical constant at least one,
is already greater than one, and the statement is deterministic. Otherwise
apply Lemma~\ref{att:lem:indicator-brackets} with \(\varepsilon=n^{-3}\).
The logarithm of twice the number of bracket endpoints is at most
\(Cq\log(2en)\), since \(q\le n\).
For each endpoint set \(D\), the scalar Bernstein inequality for Bernoulli
variables gives
\[
 \mathbb P\left\{|Q_n(D)-Q(D)|>
       \sqrt{2Q(D)x/n}+2x/(3n)\right\}\le2e^{-x}.
\]
This inequality follows, for example, by bounding the centered Bernoulli
moment generating function by
\(\exp\{Q(D)t^2/[2(1-|t|/3)]\}\) for \(|t|<3\), then applying
Chernoff's inequality independently to the two tails.
Take \(x=Cb_{n,q,\delta}\) and use a finite union bound.
On the resulting measurable event, for \(B^-\subseteq B\subseteq B^+\),
\[
 Q_n(B)-Q(B)\le
    \sqrt{2\{Q(B)+\varepsilon\}x/n}+2x/(3n)+\varepsilon,
\]
and
\[
 Q(B)-Q_n(B)\le
    \sqrt{2Q(B)x/n}+2x/(3n)+\varepsilon.
\]
Because \(\varepsilon\le b_{n,q,\delta}/n\), both right sides have the
claimed order. The finite pointwise bracket cover makes the event
simultaneous without an uncountable measurability step.
\end{proof}

\subsection{A sub-Gaussian proxy for the noise}
\begin{lemma}\label{att:lem:noise-mgf}
Condition~\eqref{att:eq:noise-square} and conditional centering imply
\eqref{att:eq:noise-mgf} with
\(v^2=8K_\xi\sigma^2/\theta_\xi\) when \(\sigma>0\).
\end{lemma}
\begin{proof}
Work under a regular conditional law given \(X\), and introduce an
independent copy \(\xi'\) under this conditional law. Jensen's inequality
and \(\E[\xi']=0\) give
\(\E[e^{t\xi}]\le\E[e^{t(\xi-\xi')}]\).
For \(m\ge1\), expansion of the exponential-square bound yields
\[
\E[|\xi|^{2m}]\le K_\xi m!(\sigma^2/\theta_\xi)^m,
\qquad
\E[|\xi-\xi'|^{2m}]
 \le 4^m K_\xi m!(\sigma^2/\theta_\xi)^m.
\]
The symmetrized variable is symmetric. Its moment-generating function is
finite for every \(t\), by the exponential-square bound and the inequality
\(|tz|\le c z^2+t^2/(4c)\). Expanding its hyperbolic cosine, using
\((2m)!\ge(m!)^2\), and applying monotone convergence gives
\[
\E[e^{t(\xi-\xi')}]
\le1+K_\xi\{\exp(4t^2\sigma^2/\theta_\xi)-1\}
\le\exp(4K_\xi t^2\sigma^2/\theta_\xi).
\]
The last inequality follows from \(1+K(e^z-1)\le e^{Kz}\) for
\(K\ge1,z\ge0\), whose derivative inequality is immediate. The same
calculation works for almost every conditional law with uniform constants.
\end{proof}

\subsection{Uniform noise scores}
\begin{lemma}\label{att:lem:scores}
Under \eqref{att:eq:noise-mgf}, with probability at least \(1-Cn^{-2}\),
simultaneously over all rectangles \(A\) with positive empirical count and
all finite transform parameters \(c_j,r_j>0\), the empirically centered
clipped-coordinate features satisfy
\begin{equation}\label{att:eq:score-bound}
\left\|N_n(A)^{-1}\sum_{i\in I_A}Z_i\xi_i\right\|_\infty
\le Cv\sqrt{a_n/N_n(A)},\qquad
|\overline\xi_A|\le Cv\sqrt{a_n/N_n(A)}.
\end{equation}
In particular, this holds for the query-specific local-range features in
\eqref{att:eq:features}.
\end{lemma}
\begin{proof}
Condition on all covariates. The noises are independent under the resulting
conditional law, although their conditional distributions can differ.
Rectangles induce at most \((n+1)^{2p}\) different subsets of the sample:
in each coordinate an interval selects a consecutive block of ordered
values, and intersections give the bound. For a fixed subset of size \(m\),
\eqref{att:eq:noise-mgf} gives
\[
P\left(\left|\sum_{i\in B}\xi_i\right|>v\sqrt{2mt}\,
          \middle|\,X_1,\ldots,X_n\right)\le2e^{-t}.
\]
Take \(t=C a_n\). A union bound, with \(C\) numerical and sufficiently
large, gives conditional failure probability at most \(Cn^{-2}\) for
\begin{equation}\label{att:eq:noise-rectangle-sums}
\left|\sum_{i:X_i\in B}\xi_i\right|
\le Cv\sqrt{N_n(B)a_n}\quad\text{for every rectangle }B.
\end{equation}
Empty subsets contribute zero. Integrating over the covariates gives the
same unconditional bound.

For \(T_j(t)=[(t-c_j)/r_j]_{[0,1]}\), its threshold-integral representation
and \eqref{att:eq:noise-rectangle-sums} imply
\begin{align*}
\left|N^{-1}\sum_{i\in I_A}T_j(X_{ij})\xi_i\right|
&\le N^{-1}\int_0^1
 \left|\sum_{i:X_i\in A,\ X_{ij}>c_j+r_ju}\xi_i\right|\,du\\
&\le Cv\sqrt{a_n/N}.
\end{align*}
The empirical mean of \(T_j\) is at most one. Subtracting that mean times
\(\overline\xi_A\), itself bounded by
\eqref{att:eq:noise-rectangle-sums}, proves the centered score assertion.
This is a deterministic implication on the same event for all transforms;
there is no conditioning on the response-selected node or on the query.
\end{proof}

\subsection{Bounded-signal split comparison}
\label{att:subsec:gain-proof}
The next lemma is the extension of the published empirical--population gain
comparison that is needed here. Its scale uses globally mass-weighted gains.
In particular, its error is not divided by a potentially small parent mass.
It parallels \citet[Lemmas~A.1--A.4]{MazumderWang2023}, but we prove it for
arbitrary covariate laws and conditionally sub-Gaussian noise rather than
invoking their density and independent bounded-noise assumptions.

\begin{lemma}[Uniform square-root gain comparison]
\label{att:lem:gain-comparison}
Let \(f\) have a bounded representative, set
\(m_0=\E[f(X)]\), \(B_f=\|f-m_0\|_\infty\), and suppose
\(\E[e^{t\xi}\mid X]\le e^{v_0^2t^2/2}\), allowing \(v_0=0\).
Let \(\widehat\Delta(A,t)\) be the empirical response gain normalized by
\(1/n\). A split with an empirically empty child has gain zero; a parent
with zero empirical count also has gain zero. Then, on an event of probability
at least \(1-Cn^{-2}\), simultaneously over all rectangles and splits,
\begin{equation}\label{att:eq:sqrt-gain}
 \left|\sqrt{\widehat\Delta(A,t)}-\sqrt{\Delta(A,t)}\right|
 \le C(B_f+v_0)\sqrt{a_n/n}.
\end{equation}
The same event may include the uniform count and rectangular noise bounds
used below.
\end{lemma}
\begin{proof}
Put \(W=f(X)-m_0\), \(t_n=a_n/n\). Apply
Lemma~\ref{att:lem:relative-counts}, with failures of order \(n^{-2}\),
to the covariate law and to the joint law of \((X,W)\) in
\(\mathbb R^{p+1}\). The complexity of the latter is at most a numerical
multiple of \(a_n\). In addition to the count bounds, this gives
\[
 |(P_n-P)(\mathbf1_A\mathbf1_{\{W>u\}})|
 \le C\{\sqrt{\mu(A)t_n}+t_n\}
 \quad\text{for every }A,u.
\]
The representation
\(W=-B_f+\int_{-B_f}^{B_f}\mathbf1_{\{W>u\}}\,du\)
therefore implies
\begin{equation}\label{att:eq:bounded-moment}
 |(P_n-P)(W\mathbf1_A)|
 \le CB_f\{\sqrt{\mu(A)t_n}+t_n\}.
\end{equation}
If \(B_f=0\) this holds trivially. Here \(P_n\) and \(P\) denote
empirical and population integration of functions of the observations.

For clarity write \(x=\mu(A)\), \(y=\mu_n(A)\), and
\(m_W(A)=\E[W\mid X\in A]\) when \(x>0\), with zero value
when \(x=0\). Also set empirical means equal to zero for empty sets.
The count bound yields
\begin{equation}\label{att:eq:sqrt-mass}
 |\sqrt y-\sqrt x|\le C\sqrt{t_n}.
\end{equation}
Indeed, if \(x<t_n\), both square roots are \(O(\sqrt{t_n})\).
If \(x\ge t_n\), divide the count error by
\(\sqrt y+\sqrt x\ge\sqrt x\).
There is also the uniform mean bound
\begin{equation}\label{att:eq:weighted-bounded-mean}
 \sqrt{\mu_n(A)}|\overline W_A-m_W(A)|\le CB_f\sqrt{t_n}.
\end{equation}
To verify it, first suppose \(x\ge C_0t_n\) for a sufficiently large
numerical \(C_0\). The count event gives \(x/2\le y\le2x\), and
\[
 \overline W_A-m_W(A)
 =\frac{(P_n-P)(W\mathbf1_A)-m_W(A)(y-x)}y.
\]
Use \eqref{att:eq:bounded-moment}, the count bound, and
\(|m_W(A)|\le B_f\), and then multiply by \(\sqrt y\), to obtain
\(CB_f(\sqrt{t_n}+t_n/\sqrt x)\le CB_f\sqrt{t_n}\).
If \(x<C_0t_n\), then \(y\le C_1t_n\) and the difference of the two
bounded means is at most \(2B_f\). This also includes zero population
mass; if \(y=0\), the weighted expression is zero by definition.

The rectangular noise bound proved in Lemma~\ref{att:lem:scores}, with
proxy \(v_0\), gives
\begin{equation}\label{att:eq:weighted-noise-mean}
 \sqrt{\mu_n(A)}|\overline\xi_A|\le Cv_0\sqrt{t_n}
 \quad\text{for every }A.
\end{equation}
When \(v_0=0\), all noise values vanish almost surely and this assertion
is deterministic. Intersect the finite collection of events just used.

For a child \(B\) of \(A\), define
\[
 d_B=\sqrt{\mu(B)}\{m_W(B)-m_W(A)\},\qquad
 \widehat d_B=\sqrt{\mu_n(B)}
     \{\overline{(Y-m_0)}_B-\overline{(Y-m_0)}_A\}.
\]
A coordinate with zero multiplying mass is interpreted as zero.
The difference is
\begin{align*}
 \widehat d_B-d_B
 &=\sqrt{\mu_n(B)}
     \{\overline W_B-m_W(B)-\overline W_A+m_W(A)
                       +\overline\xi_B-\overline\xi_A\}\\
 &\quad+\{\sqrt{\mu_n(B)}-\sqrt{\mu(B)}\}
                 \{m_W(B)-m_W(A)\}.
\end{align*}
The child terms are bounded by \eqref{att:eq:weighted-bounded-mean} and
\eqref{att:eq:weighted-noise-mean}. The parent terms have the same bounds
because \(\mu_n(B)\le\mu_n(A)\). The last term is bounded by
\eqref{att:eq:sqrt-mass} and \(|m_W(B)-m_W(A)|\le2B_f\).
Thus \(|\widehat d_B-d_B|\le C(B_f+v_0)\sqrt{t_n}\).
Population and empirical variance decompositions give
\[
 \sqrt{\Delta(A,t)}=\|(d_{A_t^-},d_{A_t^+})\|_2,
 \qquad
 \sqrt{\widehat\Delta(A,t)}
       =\|(\widehat d_{A_t^-},\widehat d_{A_t^+})\|_2.
\]
These identities still hold with a zero-mass or empty child. The reverse
triangle inequality proves \eqref{att:eq:sqrt-gain}.
\end{proof}

\begin{lemma}[Variation on a shallow frontier of the original tree]
\label{att:lem:frontier}
Suppose SID holds with constant \(\lambda\). Let \(\mathcal A_D\) be
the nodes that actually occur at depth \(D\) in the count-stopped tree,
and set \(S_D=\sum_{A\in\mathcal A_D}\mu(A)V(A)\). Do not carry early
terminal nodes forward into this sum. On the event of
Lemma~\ref{att:lem:gain-comparison}, simultaneously for every \(D\ge0\),
\begin{equation}\label{att:eq:frontier-bound}
 S_D\le(1-\lambda/2)^D V_0+CH^2\frac{2^Da_n}{n}.
\end{equation}
\end{lemma}
\begin{proof}
At an internal node \(A\), empirical optimality holds even in comparison
with an empirically unavailable population split: such a split has
empirical gain zero, while available gains are nonnegative. Every
available threshold has the sample partition of an observed-value candidate,
so restricting the algorithm to its finite candidate list does not change
this comparison. With \(\varepsilon=CH\sqrt{a_n/n}\), the gain lemma
and empirical optimality give
\[
 \sqrt{\Delta(A,\widehat t_A)}
 \ge\bigl(\sqrt{\lambda\mu(A)V(A)}-2\varepsilon\bigr)_+.
\]
If the population supremum is not attained, use splits approaching it and
pass to the limit in this scalar inequality. Since
\((u-v)_+^2\ge u^2/2-v^2\) for nonnegative \(u,v\),
\[
 \Delta(A,\widehat t_A)
 \ge\tfrac\lambda2\mu(A)V(A)-CH^2a_n/n.
\]
Subtract from \(\mu(A)V(A)\) and sum over internal nodes at depth
\(d\). There are at most \(2^d\) of them. Terminal nodes omitted from
the next frontier have nonnegative variation, so
\[
 S_{d+1}\le\vartheta S_d+CH^2 2^d a_n/n,
       \qquad\vartheta=1-\lambda/2.
\]
Starting from \(S_0=V_0\) and using
\(\sum_{j=0}^{D-1}\vartheta^{D-1-j}2^j
 =(2^D-\vartheta^D)/(2-\vartheta)\le2^D\)
proves the claim. No minimum child mass or variance-floor assumption is used.
\end{proof}

\subsection{From a shallow frontier to geometric purity}
\label{att:subsec:transfer-proof}
\begin{lemma}[Depth-to-neighborhood transfer]
\label{att:lem:transfer}
On a sample with distinct active coordinate values, suppose the relative
count event of Lemma~\ref{att:lem:relative-counts} holds with complexity
\(Ca_n\). Then, for every depth \(D\ge0\) and every \(r>0\),
\begin{equation}\label{att:eq:transfer}
 q_{n,\ell}\le C\frac{2^D(k+\ell+a_n)}n+2B_\partial r
          +\frac{12\Gamma^2 S_D}{\lambda_*^2r^2}.
\end{equation}
In this deterministic assertion the covariate law and its geometric
parameters are fixed; only the relative count event is imposed on the sample.
\end{lemma}
\begin{proof}
The relative count event and
\(C\sqrt{\mu(A)a_n/n}\le\mu(A)/2+C'a_n/n\) give
\begin{equation}\label{att:eq:small-node-mass}
 \mu(A)\le C\{N_n(A)+a_n\}/n\quad\text{for every rectangle }A.
\end{equation}
Any node with at least two observations has an available active-coordinate
split, because those sample coordinate values are distinct. Therefore every
terminal node has count at most \(k\). Declare a query count-bad if its
path terminates before depth \(D\), or its depth-\(D\) node has count
less than \(\ell\). Early terminal nodes number at most \(2^D-1\),
and nodes at depth \(D\) number at most \(2^D\). By
\eqref{att:eq:small-node-mass}, the count-bad query set has mass at most
\(C2^D(k+\ell+a_n)/n\). This also covers \(D=0\).

For every other query \(x\), its original-tree node \(A_D(x)\) has
at least \(\ell\) observations. Every observation inside that node has
tree distance at most its count, whereas every observation outside has
strictly greater distance. Thus the whole tie-inclusive neighborhood lies
inside \(A_D(x)\), regardless of later splits. By
\eqref{att:eq:geometry-bound},
\[
 \sum_{A\in\mathcal A_D}\mu(A)\omega(A)^2
 \le 12\Gamma^2S_D/\lambda_*^2.
\]
The mass of query points in depth-\(D\) nodes with \(\omega(A)>r\)
is at most \(12\Gamma^2 S_D/(\lambda_*^2r^2)\).

Let \(\mathcal U_r\) be the points within normalized coordinate distance
\(r\) of some internal grid cutpoint. Uniform active marginals and the
union bound give \(\mu(\mathcal U_r)\le2B_\partial r\), independent
of the locations or separations of the cutpoints. If \(x\notin\mathcal U_r\)
and \(\omega(A_D(x))\le r\), the node's active projection cannot cross
any grid hyperplane: every one of its points is within distance \(r\)
of \(x\) in every normalized active coordinate. All sample points in
the node therefore have the same grid label as the query. Deterministic
boundaries and the model's null sets are excluded on a common probability-one
event. The union of the three bad sets proves \eqref{att:eq:transfer}.
\end{proof}

\begin{proof}[Proof of Theorem~\ref{att:thm:purity} and
Corollary~\ref{att:cor:purity}]
Choose the bounded representative obtained by clipping the active arguments
to \(Q_S\) before applying \(F\). It agrees with \(f(X)\) almost surely,
is bounded at every point, and leaves the data and all population gains
unchanged on a common probability-one event. Use
Proposition~\ref{att:prop:geometry}, Lemma~\ref{att:lem:noise-mgf}, and
Lemmas~\ref{att:lem:gain-comparison}--\ref{att:lem:transfer}.
Their events can be intersected with failure probability \(Cn^{-2}\).
The frontier bound in the transfer lemma proves
\eqref{att:eq:purity-depth-radius} simultaneously for all \(D,r\).
If \(a_n/n\) is larger than a numerical constant, its first term is
already at least one after increasing \(C\), so that case is trivial.

Suppose \(z_n\le1/4\) and take the stated integer depth. Then
\[
 2^Dz_n\le z_n^{1/2},\qquad
 \vartheta^D\le\vartheta^{-1}z_n^\kappa\le2z_n^\kappa,
 \qquad
 2^Da_n/n\le z_n^{1/2}.
\]
Write \(R_n=V_0z_n^\kappa+H^2z_n^{1/2}\). Since the active slopes
are nonzero, \(V_0>0\), and thus \(R_n>0\). For \(B_\partial>0\),
set \(r=(\Gamma^2R_n/(B_\partial\lambda_*^2))^{1/3}\) in
\eqref{att:eq:purity-depth-radius}. This proves
\eqref{att:eq:proved-conditional}; taking expectations, and using
\(q_{n,\ell}\le1\) on the exceptional event, proves
\eqref{att:eq:proved-purity}. If \(B_\partial=0\), one grid cell has
full probability, so all labels agree almost surely.

For the corollary, bounded structural parameters allow a common SID
constant bounded away from zero, hence a common \(\kappa>0\), and bound
\(B_\partial\) and \(\Gamma\). Equation~\eqref{att:eq:root-range}
and the noise proxy give
\(H_n^2/\lambda_{*,n}^2\le C(V_{0,n}+\sigma_n^2)/\lambda_{*,n}^2\).
The stated assumptions imply \(z_n\to0\), so \(\Psi_n\to0\).
This proves both conditional convergence in probability and joint purity.
For fixed-sign slopes, \(\Gamma=1\) and \(\lambda\) is independent
of grid complexity; the displayed growing-boundary condition is exactly
what remains to make the second term of \(\Psi_n\) vanish.
\end{proof}

\begin{proof}[Proof of Proposition~\ref{att:prop:published-route}]
Shift the responses by \(m_0=\E[f(X)]\), which changes no CART
split, so the signal bound in the published notation is \(B_f\).
Under the additional hypotheses, Assumption~2.1 of
\citet{MazumderWang2023} holds with response bound \(U=B_f+b_\xi\),
and Proposition~\ref{att:prop:geometry} verifies their SID Assumption~2.2.
Take their comparison parameter to be one. Lemma~A.4, equations~(29)--(30),
implies, simultaneously over nodes obtainable before depth \(D\) and all
splits,
\[
 \Delta(A,t)\le2\widehat\Delta(A,t)+\varepsilon_D,
 \qquad
 \widehat\Delta(A,t)\le2\Delta(A,t)+\varepsilon_D,
 \qquad
 \varepsilon_D\le
 C U^2\frac{D\log(2np)+\log n}{n},
\]
with failure probability at most \(n^{-2}\).
Their quantities \(\overline t_1,\overline t_2\), defined in their
Appendix~A.1, have this order; \eqref{att:eq:mw-size} ensures the
smallness condition of the cited lemma. Constants may depend on the
density upper bound. A node reached after \(d<D\) coordinate splits
restricts at most \(d\) coordinates, so it is in their class
\(\mathcal A_{p,D-1}\), up to endpoint conventions. Their displayed
class uses closed intervals. For a half-open node, move each excluded endpoint
slightly inward to form closed rectangles with exactly the same sample
membership; this is possible for a fixed finite sample. Let those endpoints
tend back to their original values. The full-density law assigns zero mass
to every coordinate hyperplane, so the population masses and bounded first
moments converge. The population gains converge as well, including at a
zero-mass child, whose weighted mean-square term is at most
\(B_f^2\mu(B)\). Empirical gains are unchanged along a sufficiently
small sequence, chosen for the finitely many endpoints of the parent and
children. Thus the uniform published inequalities pass to our endpoint
conventions on the same event. Their uniformity in the maximizing split
also means that no random tie rule is needed here.

Empirical optimality and SID now give
\[
 \lambda\mu(A)V(A)
 \le2\widehat\Delta(A,\widehat t_A)+\varepsilon_D
 \le4\Delta(A,\widehat t_A)+3\varepsilon_D.
\]
This comparison also allows empirically empty competitors, whose gain is
zero. Sum the resulting contraction over internal nodes at each depth,
dropping early terminal nodes exactly as in
Lemma~\ref{att:lem:frontier}. It follows that
\[
 S_D\le(1-\lambda/4)^D V_0+C2^D\varepsilon_D.
\]
Intersect this event with the relative count event and apply
Lemma~\ref{att:lem:transfer}, optimizing its radius as above. This proves
\eqref{att:eq:mw-purity}. The only external CART inequality in this
specialization is the published Lemma~A.4; the broader
Lemma~\ref{att:lem:gain-comparison} is not used.
\end{proof}

\subsection{Ties and measurable definition}
\label{att:subsec:measurable}
\begin{lemma}[The tie-inclusive set is a node sample]
\label{att:lem:ties}
For every fitted tree and query \(x\),
\(I_{A_\ell(x)}=\mathcal I_{n,\ell}(x)\).
\end{lemma}
\begin{proof}
If \(A=A_\ell(x)\) is terminal, every training observation in \(A\) has
distance \(N_n(A)\), and all observations outside it have larger distance.
Since \(N_n(A)\ge\ell\), the assertion follows. Otherwise its child
\(C\) containing \(x\) has \(N_n(C)<\ell\), by the definition of the
deepest eligible node. All observations in \(C\) have distance at most
\(N_n(C)\); those in its sibling have distance exactly \(N_n(A)\).
Every observation outside \(A\) has a proper ancestor as smallest common
node, of strictly greater count. There are fewer than \(\ell\) observations
with distance below \(N_n(A)\) and at least \(\ell\) with distance at
most \(N_n(A)\). This proves equality of the two index sets.
\end{proof}

The canonical finite tree, its order statistics, and all sample means and
ranges are Borel functions of \((\mathcal D_n,x)\). Within each choice of
omitted columns, profiling the intercept gives a continuous finite-
dimensional convex objective in \(\theta\) with penalty \(\tau_A>0\).
Its sublevel sets are compact, since \(\tau_A\|\theta\|_1\) is bounded
there. To verify measurability of the selected minimizer, add
\(\varepsilon\|\theta\|_2^2\), for \(\varepsilon>0\), to the profiled
objective. The minimizer is unique and is a Borel function of the data: this
follows, for example, by approximating the minimum on rational grids and
using strict convexity and compact sublevel sets. As
\(\varepsilon\downarrow0\), these minimizers converge to the unique
minimum-Euclidean-norm point in the convex compact minimizer set of the
original objective. Every subsequential limit minimizes that objective, and
comparison with any original minimizer bounds the limit's Euclidean norm
by that minimizer's norm. Strict convexity of the norm squared makes the
limit unique. A pointwise limit of Borel functions is Borel. There are only
finitely many possible omitted-coordinate sets for each \(p\), so this
also verifies measurability across zero-range cases. All conditional risks
and bad-query probabilities in the statements are therefore well defined.

\subsection{Uniform conditional counts and bounded moments}
Let \(P_A\) be the conditional covariate law given \(X\in A\), and
\(P_{n,A}\) the empirical law of selected observations in \(A\).
The supremum statements below have measurable versions; rectangle
indicators are countably determined by one-sided rational endpoint
approximation. They will be applied on a single event valid for all
rectangles, not after conditioning on a response-selected tree.

\begin{lemma}[Counts and covariance after local truncation]
\label{att:lem:count-moments}
With probability at least \(1-Cn^{-2}\), every rectangle \(A\) with
\(N=N_n(A)\ge Ca_n\) has positive population mass and satisfies
\begin{equation}\label{att:eq:parent-count}
cN/n\le\mu(A)\le CN/n.
\end{equation}
On the same event, writing \(\gamma_A=\sqrt{a_n/N}\),
\begin{equation}\label{att:eq:conditional-count}
\sup_{B\subseteq A,\ B\text{ a rectangle}}
|P_{n,A}(B)-P_A(B)|\le C\gamma_A.
\end{equation}
For arbitrary coordinates and arbitrary finite \(c_j\), \(r_j>0\), let
\(T_j(t)=[(t-c_j)/r_j]_{[0,1]}\). Simultaneously over these choices,
\begin{equation}\label{att:eq:covariance-transfer}
|\operatorname{Cov}_{n,A}(T_j(X_j),T_h(X_h))
 -\operatorname{Cov}_A(T_j(X_j),T_h(X_h))|
\le C\gamma_A.
\end{equation}
The two coordinates can be equal.
\end{lemma}
\begin{proof}
Apply Lemma~\ref{att:lem:relative-counts} with \(q=p\) and failure probability
\(n^{-2}/4\). Its complexity is at most a constant multiple of \(a_n\).
Put \(t=a_n/n\), \(x=\mu(A)\), and \(y=N/n\). The resulting event gives
\[
|y-x|\le C\{\sqrt{xt}+t\}.
\]
If \(y\ge C_0t\) with \(C_0\) sufficiently large, first absorb
\(C\sqrt{xt}\le x/2+C't\) to get \(x\le3y\), and then substitute
this into \(y\le x+C\sqrt{xt}+Ct\) to get \(x\ge y/2\).
This proves \eqref{att:eq:parent-count}, including positivity. For
\(B\subseteq A\), set \(q=\mu(B)/\mu(A)\) and write
\[
\frac{\mu_n(B)}{\mu_n(A)}-q
=\frac{\mu_n(B)-\mu(B)-q\{\mu_n(A)-\mu(A)\}}{\mu_n(A)}.
\]
Since \(\mu(B)\le\mu(A)\), the numerator is bounded by
\(C\{\sqrt{\mu(A)t}+t\}\). Divide by \(\mu_n(A)\asymp\mu(A)\)
and use \(a_n/N\le C_0^{-1}\) to obtain
\eqref{att:eq:conditional-count}. The count lemma covers all endpoint conventions directly.

For a bounded transform \(T_j\),
\[
T_j(X_j)=\int_0^1\mathbf1_{\{X_j>c_j+r_ju\}}\,du.
\]
Every threshold event intersected with \(A\) is a rectangle.
Integrating \eqref{att:eq:conditional-count} thus gives
\(|P_{n,A}T_j-P_AT_j|\le C\gamma_A\).
The product has the representation
\[
T_j(X_j)T_h(X_h)
=\int_0^1\!\int_0^1
\mathbf1_{\{X_j>c_j+r_ju,\ X_h>c_h+r_hv\}}\,du\,dv.
\]
This remains a rectangle event when \(j=h\). Its raw empirical and
population second moments therefore differ by at most \(C\gamma_A\).
Subtract the product of first moments, all of which are bounded by one,
to obtain \eqref{att:eq:covariance-transfer}. No discretization of the
transform parameters is needed; the count event is already simultaneous.
\end{proof}

\subsection{The active Gram and inactive correlations}
\begin{lemma}[Design regularity is a consequence of the product model]
\label{att:lem:gram}
On the event of Lemma~\ref{att:lem:count-moments}, intersected with the
probability-one event that active sample coordinates are distinct, the
following holds for every rectangle \(A\) with \(N_n(A)\ge\ell\) and
every query \(x\in A\) in the covariate full-measure set. Form the local
ranges and features from \eqref{att:eq:ranges}--\eqref{att:eq:features}.
All active ranges are positive. For
\(G=N^{-1}\sum_{i\in I_A}Z_iZ_i^\mathsf T\),
\begin{align}
|G_{jh}|&\le C\gamma_A &&(j\ne h),
\label{att:eq:gram-offdiag}\\
|G_{jj}-1/12|&\le C\gamma_A &&(j\in S),
\label{att:eq:gram-diag}\\
r_j/L_j&\ge cN/n &&(j\in S).
\label{att:eq:range-lower}
\end{align}
If \(\ell\ge C_0s^2a_n\), with sufficiently large \(C_0\), then
\begin{equation}\label{att:eq:gram-inverse}
\|G_{SS}^{-1}\|_\infty\le24,
\qquad
\|G_{S^cS}G_{SS}^{-1}\|_\infty\le1/4.
\end{equation}
In these expressions omitted zero-range columns are assigned zero cross
entries. No positive lower bound on inactive variances is required.
\end{lemma}
\begin{proof}
Take \(T_j\) in Lemma~\ref{att:lem:count-moments} with \(c_j\) equal
to the local minimum and \(r_j\) the local range. On all selected
observations, \(T_j(X_{ij})=(X_{ij}-c_j)/r_j\); its empirical centering
is exactly \(Z_{ij}\). For a fixed rectangle \(A\), the population
conditional coordinates remain independent. Thus
\(\operatorname{Cov}_A(T_j,T_h)=0\) for \(j\ne h\), proving
\eqref{att:eq:gram-offdiag} simultaneously for all transform parameters,
including those chosen using the responses and query.

Fix an active coordinate. Let the conditional interval have endpoints
\(u_j<v_j\), with length \(w_j=v_j-u_j\). Denote the selected minimum
and maximum by \(m_j^X,M_j^X\). The rectangle
\(A\cap\{X_j\le m_j^X\}\) has exactly one observation, since active
coordinates are distinct. Uniformity and
\eqref{att:eq:conditional-count} give
\[
(m_j^X-u_j)/w_j\le N^{-1}+C\gamma_A.
\]
The rectangle \(A\cap\{X_j>M_j^X\}\) has no observations, and hence
\((v_j-M_j^X)/w_j\le C\gamma_A\). Endpoint conventions do not affect
the uniform conditional law. Because \(x_j\in[u_j,v_j]\),
\begin{equation}\label{att:eq:range-comparable}
(1-C\gamma_A)w_j\le M_j^X-m_j^X\le r_j\le w_j.
\end{equation}
We used \(N^{-1}\le\gamma_A\). Taking the constant in the count cutoff
large enough makes the left side positive.

The transform \((X_j-u_j)/w_j\) under \(P_A\) is uniform on \([0,1]\).
The variance part of \eqref{att:eq:covariance-transfer} yields empirical
variance \(1/12+O(\gamma_A)\). Consequently
\[
G_{jj}=\frac{w_j^2}{r_j^2}\{1/12+O(\gamma_A)\}
       =1/12+O(\gamma_A),
\]
using \eqref{att:eq:range-comparable}. This proves
\eqref{att:eq:gram-diag}. The product design gives
\(\mu(A)\le w_j/L_j\), since every other factor in its probability is at
most one. Combine this with \eqref{att:eq:parent-count} and
\eqref{att:eq:range-comparable} to prove \eqref{att:eq:range-lower}.

Finally,
\[
\|G_{SS}-(1/12)I_s\|_\infty\le Cs\gamma_A,
\qquad \|G_{S^cS}\|_\infty\le Cs\gamma_A.
\]
Choose \(C_0\) so that \(12Cs\gamma_A\le1/2\) for every \(N\ge\ell\).
The Neumann series gives the first bound in
\eqref{att:eq:gram-inverse}, and a further fixed enlargement of \(C_0\)
gives the second. The active symmetric Gram is then positive definite.
\end{proof}

\subsection{A deterministic lasso lemma}
\begin{lemma}[Restricted fit and strict inactive inequalities]
\label{att:lem:kkt}
Let empirical feature columns be centered and suppose
\(Y_i=\eta_0+Z_i^\mathsf T\theta_0+\xi_i\), with
\(\supp\theta_0\subseteq S\). Set
\(G=N^{-1}Z^\mathsf TZ\) and \(g=N^{-1}Z^\mathsf T\xi\).
Suppose \(G_{SS}\) is positive definite,
\(\|G_{SS}^{-1}\|_\infty\le C_A\),
\(\|G_{S^cS}G_{SS}^{-1}\|_\infty\le1/4\), and
\(\|g\|_\infty\le\tau/8\). Then the lasso with penalty \(\tau>0\)
and an unpenalized intercept has a unique solution, with
\[
\widehat\theta_{S^c}=0,\qquad
\|\widehat\theta_S-\theta_{0,S}\|_\infty\le(9C_A/8)\tau.
\]
If each \(|\theta_{0j}|>(9C_A/8)\tau\) for \(j\in S\), the signs are
also recovered. No minimum signal is needed for the first two conclusions.
\end{lemma}
\begin{proof}
Profile the intercept, which equals \(\overline Y\), and solve the
restricted lasso on \(S\). Its strictly convex quadratic part gives a
unique restricted solution. For an active subgradient \(z_S\) with
\(\|z_S\|_\infty\le1\), its optimality equation is
\[
G_{SS}(\widehat\theta_S-\theta_{0,S})=g_S-\tau z_S.
\]
Here empirical centering of \(Z\) makes centering of \(\xi\) immaterial.
Multiplication by the inverse gives the stated error bound. Extend this
candidate by zero outside \(S\). Write
\(M=G_{S^cS}G_{SS}^{-1}\). The inactive residual score is
\[
g_{S^c}-Mg_S+\tau Mz_S,
\]
whose infinity norm is at most
\(\tau/8+\tau/32+\tau/4=13\tau/32<\tau\).
Thus all full optimality conditions hold with strict inactive inequalities.

For completeness, any other minimizer must have the same fitted values,
by strict convexity of squared loss in the fitted vector. The supporting
subgradient just constructed has absolute value less than one on every
inactive coordinate, so equality in the objective's convex lower bound
forces every such coordinate of any minimizer to be zero. Positive
definiteness on \(S\) then forces equality of the active vectors. The
intercept is the same profiled mean. This proves uniqueness. The signal
bound makes the error smaller than every true active magnitude and gives
sign recovery. This is the usual primal--dual witness mechanism
\citep{ZhaoYu2006,Wainwright2009}, applied to a restricted minimizer without
postulating its signs in advance.
\end{proof}

\subsection{Proof of the main local theorem}
\begin{proof}[Proof of Theorem~\ref{att:thm:local}]
Intersect the events of Lemmas~\ref{att:lem:count-moments} and
\ref{att:lem:scores}, the probability-one active distinctness event, and
the probability-one event excluding all deterministic model null sets from
the sample. Their intersection has probability at least \(1-C_1n^{-2}\)
after enlarging the constant. It does not depend on the query. The count
lemma also implies that every candidate neighborhood has positive
population mass, since its count is at least \(\ell\).

Fix any pure query \(x\) in the stated full-measure set. By
Lemma~\ref{att:lem:ties}, its selected observations are exactly the sample
in \(A=A_\ell(x)\). They and the query share an affine formula with
coefficients \((a_{g(x)},\beta_{g(x)})\). The finite-dimensional model
on these observations can be written
\[
Y_i=\eta_0+\sum_{j\in S}\theta_{0j}Z_{ij}+\xi_i,
\quad
\theta_{0j}=\beta_{g(x),j}r_j,
\quad
\eta_0=a_{g(x)}+\sum_{j\in S}\beta_{g(x),j}\overline X_{Aj}.
\]
Lemma~\ref{att:lem:gram} gives the deterministic Gram hypotheses of
Lemma~\ref{att:lem:kkt}. Choose \(K_0\) so that
Lemma~\ref{att:lem:scores} gives \(\|g\|_\infty\le\tau_A/8\).
The deterministic lemma proves uniqueness,
\eqref{att:eq:no-false-positives}, and \eqref{att:eq:scaled-error}.

Since the features are empirically centered,
\(\overline Y_A=\eta_0+\overline\xi_A\). The query is in the same
piece, so
\[
\widehat f(x)-f(x)=\overline\xi_A+
 \sum_{j\in S}(\widehat\theta_j-\theta_{0j})z_j(x).
\]
Each \(|z_j(x)|\le1\). Use the score bound and the active infinity error
to prove \eqref{att:eq:query-error}. The range lower bound gives
\[
L_j|\widehat\beta_j-\beta_{g(x),j}|
\le C(n/N)\tau_A
\le Cv(n/\ell)\sqrt{a_n/\ell},
\]
which is \eqref{att:eq:coefficient-error}.
Finally, \(|\theta_{0j}|=|\beta_{g(x),j}|r_j
\ge c\lambda_*N/n\). Condition~\eqref{att:eq:finite-beta-min} ensures
that this exceeds the deterministic infinity error for every \(N\ge\ell\),
and proves \eqref{att:eq:signs}.

All steps are simultaneous for pure queries on the same training event.
An independent query is impure with probability \(\pi_{n,\ell}\); a
union bound gives the joint statement. On the event, only the bad-query
set can fail, proving the conditional statement as well.
\end{proof}

\begin{proof}[Proof of Corollary~\ref{att:cor:asymptotic}]
Corollary~\ref{att:cor:purity}
gives \(\pi_{n,\ell_n}\to0\) under the stated relative-signal and
complexity hypotheses. The local theorem therefore applies with probability
tending to one. Its error bounds give the rates, while
\eqref{att:eq:information-condition} is exactly the squared asymptotic
version of \eqref{att:eq:finite-beta-min}. For absolute coefficient
consistency use \eqref{att:eq:coefficient-error}; this distinguishes a
relative signal assertion from an absolute scale assertion in changing
models. For the choice \eqref{att:eq:ell-choice},
\[
a_n/\ell_n=O(d_n^{3/4}),\qquad
n^2a_n/\ell_n^3=O(d_n^{1/4}),
\]
since \(a_n\asymp p\log n\). The ceiling is immaterial because
\(nd_n^{1/4}\to\infty\). These relations prove both displayed rates
for a fixed model and bounded \(v\). A bound on an event of probability
\(1-o(1)\) gives the claimed \(O_P\) rate even though no coefficient bound
is asserted on impure queries.
\end{proof}

\subsection{Risk and global comparison}
\begin{proof}[Proof of Theorem~\ref{att:thm:risk}]
Clipping to an interval containing \(f(x)\) cannot increase squared error.
On \(\mathcal E_n\), integrate \eqref{att:eq:query-error} over pure
queries, and bound the error by \(4M^2\) over all remaining queries.
This proves \eqref{att:eq:risk-conditional}. On \(\mathcal E_n^c\), the
same bound \(4M^2\) holds everywhere; taking expectations proves
\eqref{att:eq:risk-expectation}. Conditional centering and homoskedasticity
of the fresh noise give, pointwise in \(x\),
\[
\E[(Y-\widehat f^M(x))^2\mid X=x,\mathcal D_n]
=\sigma^2+(f(x)-\widehat f^M(x))^2.
\]
Integrate to identify the conditional risk. Markov's inequality gives its
convergence in probability from convergence of its nonnegative expectation.

For a fixed model, any fixed supremum \(M\) is eventually below
\(M_n=\sqrt{\log(e/z_n)}\) if \(z_n\to0\). With bounded model
parameters and relative signal, \eqref{att:eq:proved-purity} gives
\[
 \pi_{n,\ell}\le Cn^{-2}
             +C\{z_n^{1/2}+z_n^{\kappa/3}+z_n^{1/6}\}.
\]
Every displayed power has a positive exponent, so multiplication by
\(\log(e/z_n)\) still gives a quantity tending to zero. Also
\(z_n\ge a_n/n\ge c\log n/n\), so \(M_n^2=O(\log n)\) and the
exceptional-sample term \(M_n^2n^{-2}\) tends to zero. The same risk proof
therefore applies with \(M_n\). This is an asymptotic device for a fixed
model, not a uniform bound over unbounded intercept sequences.
\end{proof}

\begin{proof}[Proof of Proposition~\ref{att:prop:global-gap}]
Fix the training sample and a global fitted affine function. Conditioning
on \(X_{S^c}\) makes its contribution a constant intercept in an affine
function of \(X_S\). Independence preserves the active covariate law.
For every fixed inactive value, the expected squared approximation error
is at least \(\Delta_f\), by its definition. Integration over the
inactive coordinates proves the lower bound; this argument uses
nonnegative integrands and is valid without inactive moments. Adding the
noise variance gives \eqref{att:eq:global-risk-lower}.

The affine subspace generated by \(1\) and the active coordinates is
finite-dimensional and closed in \(L^2\); the active uniforms have a
positive definite centered covariance matrix. If \(\Delta_f=0\), \(F\)
is affine almost everywhere on the full box. Continuity then makes it
affine everywhere, since any nonzero discrepancy is nonzero on a relatively
open set of positive active volume. Thus a fixed non-affine \(F\) has
\(\Delta_f>0\). By Theorem~\ref{att:thm:risk}, its local excess risk is
less than \(\Delta_f\) with probability tending to one. This proves
\eqref{att:eq:global-comparison}; the changing-gap conclusion follows by
Markov's inequality applied after division by \(\Delta_{f,n}\).

For \eqref{att:eq:additive-gap}, center each \(g_j(X_j)\). Independence
makes the centered components mutually orthogonal, and makes the centered
coordinate \(X_j-\E[X_j]\) orthogonal to every other component.
Projecting each component onto its own centered coordinate gives slope
\(\operatorname{Cov}(X_j,g_j(X_j))/\Var[X_j]\).
Orthogonality makes the residual variances add, giving the formula.
\end{proof}

\begin{proof}[Verification of Example~\ref{att:ex:vee}]
The two affine formulas are \(1/2-x_1\) and \(x_1-1/2\), which agree at
\(x_1=1/2\); the local theorem and corollary therefore apply with the
listed parameters. Symmetry gives
\(\E[|X_1-1/2|]=1/4\) and
\(\operatorname{Cov}(X_1,|X_1-1/2|)=0\). The centered second moment is
\[
\E[(X_1-1/2)^2]-(1/4)^2=1/12-1/16=1/48.
\]
Every inactive covariance is zero by independence. This proves the
population affine target and gap.

For the stated global lasso, profile its unpenalized intercept. At zero
slopes, the score in coordinate \(j\) is
\[
H_j=n^{-1}\sum_i(X_{ij}-\overline X_j)(Y_i-\overline Y).
\]
Write \(Y_i=f(X_i)+\xi_i\). The variables
\((X_{ij}-1/2)(f(X_i)-1/4)\) are bounded and mean zero, including for
\(j=1\) by symmetry. Hoeffding's inequality \citep{Hoeffding1963} and a union bound give an
\(O(\sqrt{\log(pn)/n})\) bound on their averages uniformly in \(j\),
with failure probability at most \(Cn^{-2}\). The empirical-centering
correction is a product of bounded-variable sample-mean errors and has no
larger order. Conditionally on the covariates,
\(n^{-1}\sum_i(X_{ij}-\overline X_j)\xi_i\) is sub-Gaussian with scale
at most \(v/\sqrt n\); the same union bound applies. Thus
\[
\max_j|H_j|\le C_2(1+v)\sqrt{\log(pn)/n}
\]
with probability at least \(1-Cn^{-2}\). Choose the penalty constant
strictly larger than \(C_2\). Zero slopes then satisfy every inactive
optimality inequality strictly and are the unique fitted slopes. The
prediction is \(\overline Y\), whose risk is
\(\sigma^2+1/48+(\overline Y-1/4)^2\). Its last term tends to zero by
the law of large numbers. The local risk assertion follows from
Theorem~\ref{att:thm:risk}, for example with clipping bound \(M=1/2\).
\end{proof}

\section{Scope and limitations}
\label{att:sec:limits}
The SID condition and its empirical split-comparison analysis come from the published CART literature \citep{ChiEtAl2022,MazumderWang2023}. The broader bounded-signal comparison, the depth-to-neighborhood transfer, the product-model verification of the local design, and the lasso argument are proved in this appendix.

The results rest on two structural assumptions. First, the features are independent and the active ones are uniform, so that within any rectangle the active features are nearly uncorrelated with variance close to $1/12$ (Lemma~\ref{att:lem:gram}). Second, the regression function is continuous and piecewise affine on rectangles, with the same support on every piece. Continuity then makes it additive in the active features (Proposition~\ref{att:prop:geometry}), and every active slope is nonzero on every piece. The analysis therefore does not cover interactions between features, pieces with different supports, or dependent features.

The statements are for the hard-attention estimator of Section~\ref{att:sec:model}: neighborhoods that keep all ties, an unpenalized intercept, locally range-normalized features, a penalty built from a valid upper bound on the noise scale, and, for integrated risk, clipped predictions. Forest-proximity softmax weights, a cross-validated penalty, and blending with a global model would need separate analysis. Prediction error, coefficient error, and sign recovery are separate guarantees with separate conditions: sign recovery needs the local-information condition \eqref{att:eq:finite-beta-min}, and prediction does not. The improvement over a global affine fit is a statement about integrated risk when the regression function is not affine.

%% file: kernel.tex
\section{Comparison of simplified softmax attention and Gaussian kernel regression weights}
\label{sec:softmax_gaussian}

 Here, we discuss similarities and differences between softmax attention and Gaussian kernel regression weighting. Let \( \bm{X} \in \mathbb{R}^{n \times p} \) have rows \( \bm{x}_1, \dots, \bm{x}_n \). 
 
We discuss here a very simplified form of attention weights, namely $\text{softmax}\left(\bm{X X}^T\right)$. This is simpler than even our simplest approach to attention, which weights the inner product by a diagonal matrix using the absolute value of the ridge coefficients and therefore incorporates information from the relationship $\bm{y} | \bm{X}$.

\subsection{Simplified attention weights}

We first form the score matrix
\[
\bm{S} = \bm{XX}^\top, \qquad \bm{S}_{ij} = \bm{x}_i^\top \bm{x}_j,
\]
and apply a row-wise softmax (optionally with temperature parameter \( \tau \)) to $\bm{S}$ to obtain weights $\bm{w}$ such that:
\[
\bm{w}_{ij}^{\text{soft}} 
= \frac{\exp\!\left( \frac{\bm{x}_i^\top \bm{x}_j}{\tau} \right)}
       {\sum_{k=1}^n \exp\!\left( \frac{\bm{x}_i^\top \bm{x}_k}{\tau} \right)}.
\]
The weight vector $\bm{w}_i$ defines a probability distribution over all \(j\).  
\subsection{Gaussian kernel regression weights}

In Nadaraya--Watson or Gaussian kernel regression with bandwidth \(\sigma\), we use
\[
K_\sigma(\bm{x}_i, \bm{x}_j)
= \exp\!\left( -\frac{\|\bm{x}_i - \bm{x}_j\|^2}{2\sigma^2} \right),
\]
and define normalized weights
\[
\bm{w}_{ij}^{\text{Gauss}}
= \frac{\exp\!\left( -\frac{\|\bm{x}_i - \bm{x}_j\|^2}{2\sigma^2} \right)}
       {\sum_{k=1}^n \exp\!\left( -\frac{\|\bm{x}_i - \bm{x}_k\|^2}{2\sigma^2} \right)}.
\]

We expand the squared distance and substitute it into the Gaussian kernel:
\begin{align}
\|\bm{x}_i - \bm{x}_j\|^2 &= \|\bm{x}_i\|^2 + \|\bm{x}_j\|^2 - 2\bm{x}_i^\top \bm{x}_j \\
\exp\!\left( -\frac{\|\bm{x}_i - \bm{x}_j\|^2}{2\sigma^2} \right)
&= \exp\!\left(-\frac{\|\bm{x}_i\|^2}{2\sigma^2}\right)
  \exp\!\left(-\frac{\|\bm{x}_j\|^2}{2\sigma^2}\right)
  \exp\!\left(\frac{\bm{x}_i^\top \bm{x}_j}{\sigma^2}\right).
\end{align}

For a fixed \(i\), the factor \(\exp\left(-\frac{\|\bm{x}_i\|^2}{(2\sigma^2)}\right)\) is constant in \(j\)
and cancels during normalization, giving (up to normalization)
\[
w_{ij}^{\text{Gauss}} \propto
\exp\!\left(-\frac{\|\bm{x}_j\|^2}{2\sigma^2}\right)
\exp\!\left(\frac{\bm{x}_i^\top \bm{x}_j}{\sigma^2}\right).
\]

\subsection{Summary}

Gaussian kernel weights and softmax attention weights are the same when (1) the temperature and bandwidth satisfy \(\tau = \sigma^2\), and (2) all data points have equal norm, \(\|\mathbf{x}_j\| = c\). In this case the factor
$
\exp\!\left(-\frac{\|\mathbf{x}_j\|^2}{2\sigma^2}\right)
$
is constant in \(j\) and cancels in the normalization, so that
\[
\textbf{w}_{ij}^{\text{Gauss}}
=
\frac{\exp\!\left( \frac{\mathbf{x}_i^\top \mathbf{x}_j}{\sigma^2} \right)}
     {\sum_k \exp\!\left( \frac{\mathbf{x}_i^\top \mathbf{x}_k}{\sigma^2} \right)}
=
\textbf{w}_{ij}^{\text{soft}}
\quad (\tau = \sigma^2).
\]\,
This situation arises, for example, when the rows are \(\ell_2\)-normalized so that \(\|\mathbf{x}_j\|=1\). 

However, in general, the two constructions are slightly different. Gaussian weights use squared Euclidean distance,
\[
w_{ij}^{\text{Gauss}} = f\!\left(\|\mathbf{x}_i - \mathbf{x}_j\|^2\right),
\]
and are invariant under rotations and translations of the input space. Softmax attention depends on inner products,
\[
w_{ij}^{\text{soft}} = g\!\left(\mathbf{x}_i^\top \mathbf{x}_j\right),
\]
and is therefore not translation invariant and more sensitive to the individual norms \(\|\mathbf{x}_i\|\) and \(\|\mathbf{x}_j\|\), allowing large-norm vectors to dominate the weights. Additionally, Gaussian weights include an additional global factor
$
\exp\!\left(-\frac{\|\mathbf{x}_j\|^2}{2\sigma^2}\right),
$
which penalizes high-norm points uniformly across all queries \(i\), whereas attention weights lack such a term and only reweight via relative inner products.
Finally, the bandwidth \(\sigma\) and temperature \(\tau\) play analogous roles by controlling how concentrated the weights are, but they do so in different geometries: \(\sigma\) sets the scale of decay in distance space, while \(\tau\) controls the sharpness of the distribution in inner-product space.

In summary, when all rows are normalized and \(\tau = \sigma^2\), the two weighting schemes are the same. Away from this regime, Gaussian kernel regression remains strictly radial in distance, whereas softmax attention depends on vector norms and loses translation invariance.